\documentclass[nohyperref]{article}

\usepackage{hyperref}       % hyperlinks

  \usepackage[final]{neurips_2023}

\usepackage{amsmath,amsfonts,bm}

\def\eqref#1{equation~\ref{#1}}
\def\1{\bm{1}}

\def\eps{{\epsilon}}

\def\vtheta{{\bm{\theta}}}

\def\vc{{\bm{c}}}

\def\vf{{\bm{f}}}

\def\vo{{\bm{o}}}

\def\vx{{\bm{x}}}

\def\vz{{\bm{z}}}

\DeclareMathAlphabet{\mathsfit}{\encodingdefault}{\sfdefault}{m}{sl}
\SetMathAlphabet{\mathsfit}{bold}{\encodingdefault}{\sfdefault}{bx}{n}

\newcommand{\E}{\mathbb{E}}

\newcommand{\R}{\mathbb{R}}

\newcommand{\Var}{\mathrm{Var}}

\DeclareMathOperator*{\argmax}{arg\,max}
\DeclareMathOperator*{\argmin}{arg\,min}

\DeclareMathOperator{\sign}{sign}

\usepackage[utf8]{inputenc} % enable UTF-8
\usepackage{lipsum} % write pseudo text

\usepackage[T1]{fontenc}

\usepackage{etoolbox}
\newbool{includeappendix}
\setbool{includeappendix}{true}

\ifdefined\isoverfull
\else
\fi

\usepackage{subcaption}

\usepackage{xcolor} % for definecolor
\definecolor{my-full-blue}{HTML}{1F77B4}
\definecolor{my-full-orange}{HTML}{FF7F0E}
\definecolor{my-full-green}{HTML}{2CA02C}
\definecolor{my-full-red}{HTML}{d62728}
\definecolor{my-full-purple}{HTML}{9467bd}
\definecolor{my-full-brown}{HTML}{8c564b}
\definecolor{my-full-pink}{HTML}{e377c2}
\definecolor{my-full-gray}{HTML}{7f7f7f}
\definecolor{my-full-olive}{HTML}{bcbd22}
\definecolor{my-full-cyan}{HTML}{17becf}
\definecolor{c1}{RGB}{86, 100, 26}
\definecolor{c2}{RGB}{192, 175, 251}
\definecolor{c3}{RGB}{230, 161, 118}
\definecolor{c4}{RGB}{0, 103, 138}
\definecolor{c5}{RGB}{152, 68, 100}
\definecolor{c6}{RGB}{94, 204, 171}
\definecolor{c7}{RGB}{205, 205, 205}

\definecolor{csabr}{RGB}{152, 68, 100}
\definecolor{cibp}{RGB}{94, 204, 171}
\definecolor{ctaps}{RGB}{0, 103, 138}

\definecolor{cm1}{HTML}{1f77b4}
\definecolor{cm2}{HTML}{ff7f0e}
\definecolor{cm3}{HTML}{2ca02c}
\definecolor{cm4}{HTML}{d62728}
\definecolor{cm5}{HTML}{9467bd}
\definecolor{cm6}{HTML}{8c564b}
\definecolor{cm7}{HTML}{e377c2}
\definecolor{cm8}{HTML}{7f7f7f}
\definecolor{cm9}{HTML}{bcbd22}
\definecolor{cm10}{HTML}{17becf}

\colorlet{my-blue}{my-full-blue!30}
\colorlet{my-orange}{my-full-orange!30}
\colorlet{my-green}{my-full-green!30}
\colorlet{my-red}{my-full-red!30}
\colorlet{my-purple}{my-full-purple!30}
\colorlet{my-brown}{my-full-brown!30}
\colorlet{my-pink}{my-full-pink!30}
\colorlet{my-gray}{my-full-gray!30}
\colorlet{my-olive}{my-full-olive!30}
\colorlet{my-cyan}{my-full-cyan!30}

\usepackage{listings}

\usepackage{textcomp}

\usepackage{xcolor}

\usepackage[scaled=0.8]{beramono}

\definecolor{ckeyword}{HTML}{7F0055}
\definecolor{ccomment}{HTML}{3F7F5F}
\definecolor{cstring}{HTML}{2A0099}

\lstdefinestyle{numbers}{
	numbers=left,
	framexleftmargin=20pt,
	numberstyle=\tiny,
	firstnumber=auto,
	numbersep=1em,
	xleftmargin=2em
}

\lstdefinestyle{layout}{
	frame=none,
	captionpos=b,
}

\lstdefinestyle{comment-style}{
	morecomment=[l]//,
	morecomment=[s]{/*}{*/},
	commentstyle={\color{ccomment}\itshape},
}

\lstdefinestyle{string-style}{
	morestring=[b]",%
	morestring=[b]',%
	stringstyle={\color{cstring}},
	showstringspaces=false,%
}

\lstdefinestyle{keyword-style}{
	keywordstyle={\ttfamily\bfseries},
	morekeywords={
		function,
		constructor,
		int,
		bool,
		return,
		returns,
		uint
	},
	morekeywords = [2]{},
	keywordstyle = [2]{\text},
	sensitive=true,
}

\lstdefinestyle{input-encoding}{
	inputencoding=utf8,
	extendedchars=true,
	literate=
	{ℝ}{$\reals$}1%
	{→}{$\rightarrow$}1%
	{α}{$\alpha$}1%
	{β}{$\beta$}1%
	{λ}{$\lambda$}1%
	{θ}{$\theta$}1%
	{ϕ}{$\phi$}1%
}

\lstdefinestyle{escaping}{
	moredelim={**[is][\color{blue}]{\%}{\%}},
	escapechar=|,
	mathescape=true
}

\lstdefinestyle{default-style}{
	basicstyle=\fontencoding{T1}\ttfamily\footnotesize,
	style=numbers,
	style=layout,
	style=comment-style,
	style=string-style,
	style=keyword-style,
	style=input-encoding,
	style=escaping,
	tabsize=2,
	upquote=true
}

\lstdefinelanguage{BASIC}{
	language=C++,
	style=default-style
}[keywords,comments,strings]%

\usepackage[capitalize,noabbrev]{cleveref}

\crefname{listing}{Lst.}{listings}
\crefname{line}{Lin.}{Lin.}
\crefname{appendix}{App.}{App.}
\crefname{lemma}{Lemma}{Lemmas}
\Crefname{lemma}{Lemma}{Lemmas}
\crefname{thm}{Theorem}{Theorems}
\Crefname{thm}{Theorem}{Theorems}

\newcommand{\app}[1]{%
	\ifbool{includeappendix}{\cref{#1}}{the appendix}%
}
\newcommand{\App}[1]{%
	\ifbool{includeappendix}{\cref{#1}}{The appendix}%
}

\usepackage{tikz}

\usetikzlibrary{calc,decorations,decorations.pathmorphing}
\usetikzlibrary{positioning,fit,arrows}
\usetikzlibrary{decorations.markings}
\usetikzlibrary{shapes,shapes.geometric}
\usetikzlibrary{shadows,patterns,snakes}
\usetikzlibrary{backgrounds,decorations.pathreplacing,automata}
\usetikzlibrary{intersections}
\usetikzlibrary{angles,quotes}
\usetikzlibrary{plotmarks}
\usetikzlibrary{patterns}
\usetikzlibrary{arrows.meta}
\usetikzlibrary{shapes.misc}
\usetikzlibrary{chains}
\usetikzlibrary{calc}

\usepackage[utf8]{inputenc} % allow utf-8 input
\usepackage[T1]{fontenc}    % use 8-bit T1 fonts
\usepackage{url}            % simple URL typesetting
\usepackage{booktabs}       % professional-quality tables
\usepackage{amsfonts}       % blackboard math symbols
\usepackage{nicefrac}       % compact symbols for 1/2, etc.
\usepackage{microtype}      % microtypography
\usepackage{graphicx}
\usepackage{adjustbox}
\usepackage{threeparttable}
\usepackage{thmtools, nameref}
\usepackage{multirow}
\usepackage{amssymb}
\usepackage{mathtools}
\usepackage{amsthm}
\usepackage[abbreviations]{foreign}
\usepackage{enumitem}
\usepackage{makecell}
\usepackage{algorithm}
\usepackage[noend]{algpseudocode} % for algorithm pseudocode
\usepackage{wrapfig}

\newtheorem{theorem}{Theorem}

\crefname{lemma}{Lemma}{lemmas}
\Crefname{lemma}{Lemma}{Lemmas}
\crefname{corollary}{corollary}{corollaries}
\Crefname{corollary}{Corollary}{Corollaries}

\DeclareMathOperator*{\relu}{ReLU}

\newcommand{\bc}[1]{\mathcal{#1}}
\newcommand{\mbf}[1]{\mathbf{#1}}

\newcommand{\bs}[1]{\boldsymbol{#1}}
\newcommand{\B}{\bc{B}}
\renewcommand{\L}{\bc{L}}

\newcommand{\w}{\vtheta}

\newcommand{\tool}{\textsc{TAPS}\xspace}

\newcommand{\toolp}{\textsc{STAPS}\xspace}

\newcommand{\mnbab}{\textsc{MN-BaB}\xspace}
\newcommand{\pgd}{\textsc{PGD}\xspace}
\newcommand{\milp}{\textsc{MILP}\xspace}

\newcommand{\crownibp}{\textsc{CROWN-IBP}\xspace}

\newcommand{\deepz}{\textsc{DeepZ}\xspace}
\newcommand{\ibp}{\textsc{IBP}\xspace}
\newcommand{\sabr}{\textsc{SABR}\xspace}
\newcommand{\ibpr}{\textsc{IBP-R}\xspace}
\newcommand{\diffai}{\textsc{DiffAI}\xspace}
\newcommand{\colt}{\textsc{COLT}\xspace}

\newcommand{\boxd}{\textsc{Box}\xspace}
\newcommand{\zono}{\textsc{Zonotope}\xspace}

\newcommand{\sortnet}{\textsc{SortNet}\xspace}

\newcommand{\cifar}{CIFAR-10\xspace}
\newcommand{\TIN}{\textsc{TinyImageNet}\xspace}

\newcommand{\mnist}{\textsc{MNIST}\xspace}

\newcommand{\cnnt}{\texttt{CNN3}\xspace}
\newcommand{\cnns}{\texttt{CNN7}\xspace}
\newcommand{\cnnss}{\texttt{CNN7}s\xspace}

\renewcommand{\th}{\textsuperscript{th}\xspace}

\usepackage{siunitx}
\newcolumntype{d}[1]{S[table-format=#1]}

\colorlet{cbackground}{c7!20}
\colorlet{cexact}{my-full-blue!45}
\colorlet{cexactlatent}{my-full-green!40}
\colorlet{cfwd}{black!100}
\colorlet{cpgd}{my-full-purple!75}
\colorlet{cpgdsignle}{my-full-red!90!black!65}

\colorlet{cbwd}{my-full-red!90!black!65}
\colorlet{cibp}{black!80}
\colorlet{ctool}{c4!100}
\colorlet{netinside}{c7!100}

\tikzstyle{toolstyle}=[dotted, line width = 1.2pt,draw=ctool]
\tikzstyle{ibpstyle}=[dash pattern=on 5pt off 2pt, cibp]
\tikzstyle{exactstyle}=[fill=cexact, opacity=1.0, draw=none]
\tikzstyle{exactlatentstyle}=[fill=cexactlatent, opacity=0.8, draw=none]
\tikzstyle{point}=[line width=0.5pt, draw=black, cross out, inner sep=0pt, minimum width=3pt, minimum height=3pt, anchor=center]
\tikzstyle{wcpoint}=[line width=0.5pt, draw=black, fill=black, circle, inner sep=0pt, minimum width=2.5pt, minimum height=2.5pt, anchor=center]
\tikzstyle{pgdarrow}=[color=cpgd, thick]
\tikzstyle{pgdsinglearrow}=[color=cpgdsignle, thick]
\tikzstyle{fwdarrow}=[color=cfwd, thick, dashed]
\tikzstyle{bwdarrow}=[color=cbwd,  line width = 1.2pt,, dotted]
\tikzstyle{pane}=[fill=cbackground, rectangle, rounded corners=2pt]

\newcommand{\newprotectedcommand}[2]{\newcommand{#1}{\protecting{#2}}}

\newprotectedcommand{\markerfwd}{\tikz[]{\draw[-stealth, fwdarrow] (0, 0) -- (0.3, 0); \node[anchor=center, minimum width=0pt, minimum height=4pt, inner sep=0pt] () at (0.15, 0.0) {};}\xspace}
\newprotectedcommand{\markerbwd}{\tikz[]{\draw[-stealth, bwdarrow] (0, 0) -- (0.3, 0); \node[anchor=center, minimum width=0pt, minimum height=4pt, inner sep=0pt] () at (0.15, 0.0) {};}\xspace}
\newprotectedcommand{\markerpgd}{\tikz[]{\draw[-stealth, pgdarrow] (0, 0) -- (0.3, 0); \node[anchor=center, minimum width=0pt, minimum height=4pt, inner sep=0pt] () at (0.15, 0.0) {};}\xspace}
\newprotectedcommand{\markerexact}{\tikz[]{\node[fill, aspect=1, inner sep=0pt, minimum size=2.1mm, exactstyle]{};}\xspace}
\newprotectedcommand{\markerexactlatent}{\tikz[]{\node[fill, aspect=1, inner sep=0pt, minimum size=2.1mm, exactlatentstyle]{};}\xspace}
\newprotectedcommand{\markeribp}{\tikz[]{\node[aspect=1, draw=cibp, inner sep=0pt, minimum size=2.1mm, ibpstyle, dash pattern=on 2.5pt off 1pt]{};}\xspace}
\newprotectedcommand{\markertool}{\tikz[]{\draw[draw=ctool, toolstyle, dash pattern=on 1pt off 0.75pt] (0,0.2) -- (0.2,0.2) -- (0.2,0);}\xspace}
\newprotectedcommand{\markerwc}{\tikz[]{\node[wcpoint] at (0,0.01) {};\node[] at (0,0) {};}\xspace}
\newprotectedcommand{\markerpoint}{\tikz[]{\node[point]{};}\xspace}
\newprotectedcommand{\markersinglepgd}{\tikz[]{\draw[pgdsinglearrow, dashed, line width =1.1pt] (0, 0) -- (0.3, 0); \node[anchor=center, minimum width=0pt, minimum height=4pt, inner sep=0pt] () at (0.15, 0.0) {};}\xspace}

\newprotectedcommand{\markermultipgd}{\tikz[]{\draw[pgdarrow, dotted, dash pattern=on 1pt off 0.75pt] (0,0.2) -- (0.2,0.2) -- (0.2,0);}\xspace}

\newprotectedcommand{\markerwcsingle}{\tikz[]{\node[point,pgdsinglearrow] at (0,0.0) {};\node[] at (0,0) {};}\xspace}
\newprotectedcommand{\markeradvpoint}{\tikz[]{\node[point,cpgd]{};}\xspace}

\newprotectedcommand{\markersabr}{\tikz[]{\node[fill, diamond, color=csabr,inner sep=0,  minimum size=2.5mm, aspect=0.5]{};}\xspace}
\newprotectedcommand{\markertaps}{\tikz[]{\node[aspect=1, fill=c4, inner sep=0pt, minimum size=2.1mm]{};}\xspace}
\newcommand{\markeribpintro}{\tikz[]{\node[fill, circle, aspect=1, color=c6, inner sep=0pt, minimum size=2.5mm]{};}\xspace}

\title{Connecting Certified and Adversarial Training}
\author{%
  Yuhao Mao, Mark Niklas Müller, Marc Fischer, Martin Vechev\\
  Department of Computer Science\\
  ETH Zurich, Switzerland\\
  \texttt{\{yuhao.mao, mark.mueller, marc.fischer, martin.vechev\}@inf.ethz.ch} \\
}

\begin{document}
\maketitle

\begin{abstract}
Training certifiably robust neural networks remains a notoriously hard problem. 
While adversarial training optimizes \emph{under-approximations} of the worst-case loss, which leads to insufficient regularization for certification, sound certified training methods, optimize loose \emph{over-approximations}, leading to over-regularization and poor (standard) accuracy. 
In this work, we propose \tool, an (unsound) certified training method that combines \ibp and \pgd training to optimize more precise, although not necessarily sound, worst-case loss approximations, reducing over-regularization and increasing certified and standard accuracies. 
Empirically, \tool achieves a new state-of-the-art in many settings, e.g., reaching a certified accuracy of $22\%$ on \TIN for $\ell_\infty$-perturbations with radius $\epsilon=1/255$. We make our implementation and networks public at \texttt{\href{https://github.com/eth-sri/taps}{github.com/eth-sri/taps}}.

\end{abstract}

% a simple test section that may be deleted
% \input{test-section}
\section{Introduction} \label{sec:introduction}
Adversarial robustness, \ie, a neural network's resilience to small input perturbations \citep{BiggioCMNSLGR13,SzegedyZSBEGF13}, has established itself as an important research area.

\paragraph{Neural Network Certification} can rigorously prove such robustness: While complete verification methods \citep{TjengXT19,BunelLTTKK20,ZhangWXLLJ22,FerrariMJV22} can decide every robustness property given enough (exponential) time, incomplete methods \citep{WongK18,SinghGPV19,ZhangWCHD18} trade precision for scalability.

\paragraph{Adversarial training} methods, such as \pgd \citep{MadryMSTV18}, aim to improve robustness by training with samples that are perturbed to approximately maximize the training loss. This can be seen as optimizing an \emph{under-approximation} of the worst-case loss. While it \emph{empirically} improves robustness significantly, it generally does not induce sufficient regularization for certification and has been shown to fail in the face of more powerful attacks \citep{TramerCBM20}.

\paragraph{Certified Training} methods, in contrast, optimize approximations of the worst-case loss, thus increasing certified accuracies at the cost of over-regularization that leads to reduced standard accuracies. 
In this work, we distinguish two certified training paradigms. 
Sound methods \citep{MirmanGV18,GowalIBP2018,ShiWZYH21} compute sound over-approximations of the worst-case loss via bound propagation. The resulting approximation errors induce a strong (over-)regularization that makes certification easy but causes severely reduced standard accuracies. Interestingly, reducing these approximation errors by using more precise bound propagation methods, empirically, results in strictly \emph{worse} performance, as they induce harder optimization problems \citep{jovanovic2022paradox}.  
This gave rise to unsound methods \citep{BalunovicV20,PalmaIBPR22,MuellerEFV22}, which aim to compute \emph{precise} but not necessarily sound approximations of the worst-case loss, reducing (over)-regularization and resulting in networks that achieve higher standard and certified accuracies, but can be harder to certify. Recent advances in certification techniques, however, have made their certification practically feasible \citep{FerrariMJV22,ZhangWXLLJ22}. 

We illustrate this in \cref{fig:intro_fig}, where we compare certified training methods with regard to their worst-case loss approximation errors and the resulting trade-off between certified and standard accuracy. On the left, we show histograms of the worst-case loss approximation error over test set samples (see \cref{sec:ablation} for more details). Positive values (right of the y-axis) correspond to over- and negative values (left of the y-axis) to under-approximations. 
As expected, we observe that the sound \ibp \citep{GowalIBP2018} always yields over-approximations (positive values) while the unsound \sabr \citep{MuellerEFV22} yields a more precise ($6$-fold reduction in mean error) but unsound approximation of the worst-case loss.
Comparing the resulting accuracies (right), we observe that this more precise approximation of the actual optimization objective, \ie, the true worst-case loss, by \sabr (\markersabr) yields both higher certified and standard accuracies than the over-approximation by \ibp (\markeribpintro). Intuitively, reducing the over-regularization induced by a systematic underestimation of the network's robustness allows it to allocate more capacity to making accurate predictions.

The core challenge of effective certified training is, thus, to compute precise (small mean error and low variance) worst-case loss approximations that induce a well-behaved optimization problem.
% However, while \sabr improves significantly on the approximation precision of \ibp, it still suffers from high variance and non-negligible error. 

\begin{wrapfigure}[16]{r}{0.537\textwidth}
    \centering
    \vspace{-4mm}
    \begin{minipage}[c]{0.65\linewidth}
    \begin{subfigure}[b]{1.02\linewidth}
        \hspace{2mm}
        \includegraphics[width=.920\textwidth]{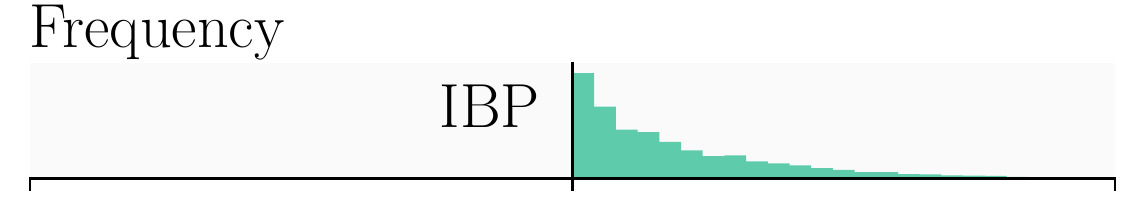}
        \vspace{-1.5mm}
    \end{subfigure}
    \begin{subfigure}[b]{1.02\linewidth}
        \hspace{2mm}
        \includegraphics[width=.920\textwidth]{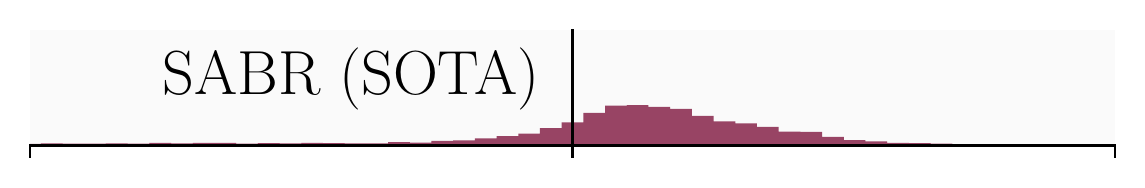}
        \vspace{-1.5mm}
    \end{subfigure}
    \begin{subfigure}[b]{1.02\linewidth}
        \centering
        \includegraphics[width=1.0\textwidth]{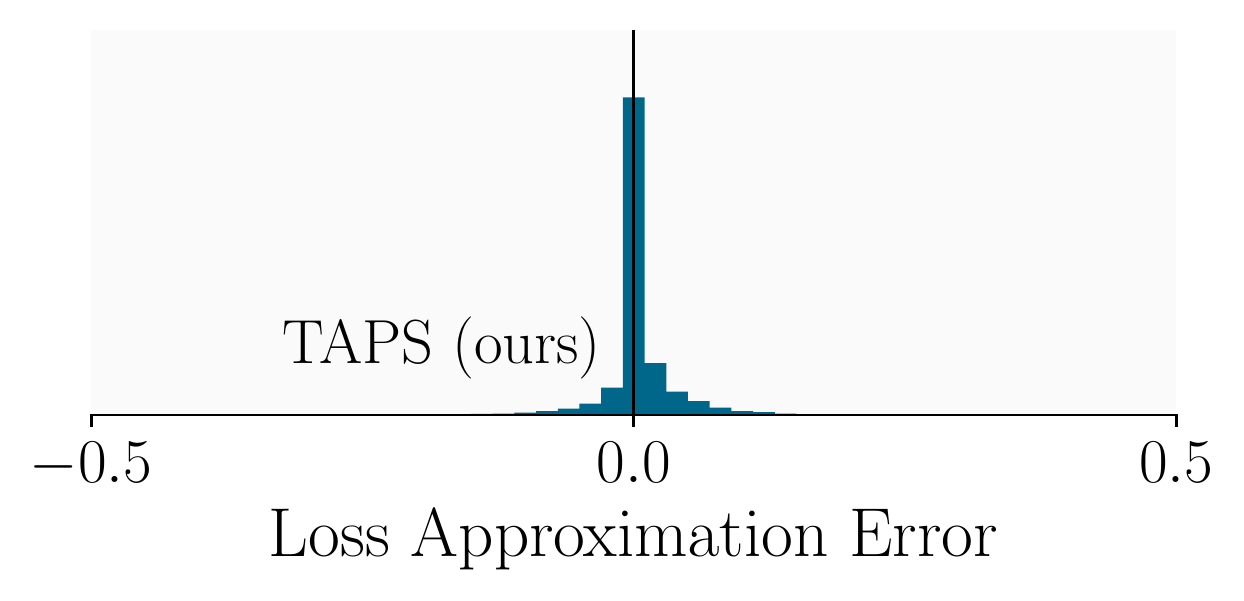}
    \end{subfigure}
    \end{minipage}
    \hfil
    \begin{minipage}[c]{0.3\linewidth}
        \centering
        \includegraphics[width=1.0\textwidth]{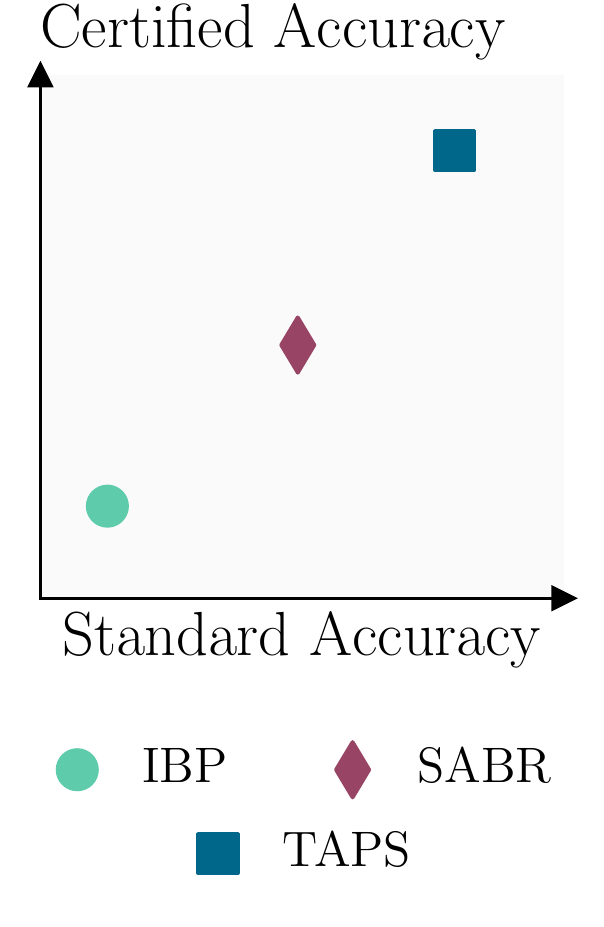}
    \end{minipage}

    \vspace{-2mm}
    \caption{Histograms of the worst-case loss approximation errors over the test set (left) for different training methods show that \tool (our work) achieves the most precise approximations and highest certified accuracy (right). Results shown here are for a small \cnnt.}
    \label{fig:intro_fig}
\end{wrapfigure}

\paragraph{This Work} proposes \textbf{T}raining via \textbf{A}dversarial \textbf{P}ropagation through \textbf{S}ubnetworks (\tool), a novel (unsound) certified training method tackling this challenge, thereby increasing both certified and standard accuracies. Compared to \sabr (\markersabr the current state-of-the-art), \tool (\markertaps) enjoys a further $5$-fold mean approximation error reduction and significantly reduced variance (\cref{fig:intro_fig} left), leading to improved certified and natural accuracies (right). The key technical insight behind \tool is to combine \ibp and \pgd training via a gradient connector, a novel mechanism that allows training the whole network jointly such that the over-approximation of \ibp and under-approximations of \pgd cancel out. We demonstrate in an extensive empirical study that \tool yields exceptionally tight worst-case loss approximations which allow it to improve on state-of-the-art results for \mnist, \cifar, and \TIN.

% % Our main contributions are:
% \begin{itemize}[noitemsep, nolistsep]
%     \item We propose the novel certified training method \tool\footnote{\textbf{T}raining via \textbf{A}dversarial \textbf{P}ropagation through \textbf{S}ubnetworks}, combining \ibp and \pgd training to yield precise yet tractable worst-case loss approximations (\cref{sec:tool}).
%     \item We theoretically analyze the slackness of \ibp approximations and under which conditions it becomes tight, motivating \tool (\cref{sec:theory}).
%     \item We show in an extensive experimental evaluation that \tool achieves and exceeds state-of-the-art performance on \mnist, \cifar, and \TIN, and conduct an in-depth ablation study (\cref{sec:eval}).
% \end{itemize}

\section{Background on Adversarial and Certified Training}
Here, we provide the necessary background on adversarial and certified training.
We consider a classifier $F \colon \bc{X} \mapsto \bc{Y}$ parameterized by weights $\w$ and predicting a class $y_\text{pred} \coloneqq F(\vx) \coloneqq \argmax_{y \in \bc{Y}} f_{y}(x)$ for every input $\vx \in \bc{X} \subseteq \R^d$ with label $y \in \bc{Y} \coloneqq \{1, \dots, K\}$ where $\vf \colon \bc{X} \mapsto \mathbb{R}^{|\bc{Y}|}$ is a neural network, assigning a numerical logit $o_i \!\coloneqq \!f_i(\vx)$ to each class $i$. 

%Commonly, such classifiers are trained on the cross-entropy loss $\bc{L}_\text{CE}(\vx, y) = -\ln \softmax_y(\vf(\vx))$.

\paragraph{Adversarial Robustness}
We call a classifier adversarially robust on an $\ell_p$-norm ball $\bc{B}_p(\vx, \epsilon)$ if it classifies all elements within the ball to the correct class, \ie, $F(\vx') = y$ for all perturbed inputs $\vx' \in \bc{B}_p(\vx, \epsilon)$.
In this work, we focus on $\ell_\infty$-robustness with $\bc{B}_\infty(\vx, \epsilon) \coloneqq \{\vx' \mid \|\vx' - \vx\|_\infty \le \epsilon\}$ and thus drop the subscript $\infty$.

\paragraph{Neural Network Certification} is used to formally \emph{prove} robustness properties of a neural network, \ie, that all inputs in the region $\bc{B}(\vx, \epsilon)$ yield the correct classification. We call samples $\vx$ where this is successfull, certifiably robust and denote the portion of such samples as \emph{certified accuracy}, forming a lower bound to the true robustness of the analyzed network.

Interval bound propagation (IBP) \citep{MirmanGV18,GowalIBP2018} is a particularly simple yet effective certification method.
Conceptually, it computes an over-approximation of a network's reachable set by propagating the input region $\bc{B}(\vx, \epsilon)$ through the network, before checking whether all reachable outputs yield the correct classification. This is done by, first, over-approximating the input region $\bc{B}(\vx, \epsilon)$ as a \boxd $[\underline{\vx}^0, \overline{\vx}^0]$ (each dimension is described as an interval), centered at $\vc^0 = \vx$ and with radius $\bs{\delta}^0 = \epsilon$, such that we have the $i$\th dimension of the input $x^0_i \in [c^0_i - \delta^0_i, c^0_i + \delta^0_i]$. 
We then propagate it through the network layer-by-layer (for more details, see \citep{MirmanGV18,GowalIBP2018}), until we obtain upper and lower bounds $[\underline{\vo}^\Delta, \overline{\vo}^\Delta]$ on the logit differences $\vo^\Delta \coloneqq \vo - o_y \mbf{1}$. If we can now show dimensionwise that $\overline{\vo}^\Delta <  0$ (except for $\overline{\vo}^\Delta_y = 0$), this proves robustness.
Note that this is equivalent to showing that the maximum margin loss $\L_{\text{MA}}(\vx', y) \coloneqq \max_{i \neq y} \overline{o}^\Delta_i$ is less than $0$ for all perturbed inputs $\vx' \in \B(\vx, \epsilon)$.

\paragraph{Training for Robustness} aims to find a model parametrization $\w$ that minimizes the expected worst-case loss for some loss-function $\bc{L}$:
\begin{align}
	&\w = \argmin_{\w} \E_{\vx, y} \left[ \max_{\vx' \in \B(\vx, \epsilon)} \L(\vx', y) \right].\label{eq:opt_robust}
	%, \text{where} \L_{\text{CE}}(\vx, y) \coloneqq \ln \big( 1 + \sum_{i \neq y} \exp(f_i(\vx) - f_y(\vx) \big). \label{eq:opt}
\end{align}
As the inner maximization objective in \cref{eq:opt_robust} can generally not be solved exactly, it is often under- or over-approximated, giving rise to adversarial and certified training, respectively.

\paragraph{Adversarial Training} optimizes a lower bound on the inner maximization problem in \cref{eq:opt_robust} by training the network with concrete samples $\vx' \in \bc{B}(\vx, \epsilon)$ that (approximately) maximize the loss function.
A well-established method for this is \emph{Projected Gradient Descent (PGD)} training \citep{MadryMSTV18} which uses the Cross-Entropy loss $\L_{\text{CE}}(\vx, y) \coloneqq \ln \big( 1 + \sum_{i \neq y} \exp(f_i(\vx) - f_y(\vx)) \big)$.
Starting from a random initialization point $\hat{\vx}_0 \in \bc{B}(\vx, \epsilon)$, it performs $N$ update steps $$\hat{\vx}_{n+1}=\Pi_{\bc{B}(\vx, \epsilon)}\hat{\vx}_n + \eta \sign(\nabla_{\hat{\vx}_n} \bc{L}_{\text{}}(\hat{\vx}_n,y))$$ with step size $\eta$ and projection operator $\Pi$.
Networks trained this way typically exhibit good empirical robustness but remain hard to formally certify and vulnerable to stronger or different attacks \citep{TramerCBM20,Croce020a}.

\paragraph*{Certified Training\normalfont{,}} in contrast, is used to train \emph{certifiably} robust networks.
In this work, we distinguish two classes of such methods: while \emph{sound} methods optimize a sound upper bound of the inner maximization objective in \cref{eq:opt_robust}, \emph{unsound} methods sacrifice soundness to use an (in expectation) more precise approximation.
Methods in both paradigms are often based on evaluating the cross-entropy loss $\L_\text{CE}$ with upper bounds on the logit differences $\overline{\vo}^\Delta$.

\ibp (a sound method) uses sound \boxd bounds on the logit differences, yielding
\begin{equation} \label{eq:ibp}
	\L_{\text{IBP}}(\vx, y, \epsilon) \coloneqq \ln \big( 1 + \sum_{i \neq y} \exp(\overline{o}^\Delta_i) \big).
\end{equation}
\sabr (an unsound method) \citep{MuellerEFV22}, in contrast, first searches for an adversarial example $\vx' \in \B(\vx', \epsilon - \tau)$ and then computes \boxd-bounds only for a small region $\B(\vx', \tau) \subset \B(\vx, \epsilon)$ (with $\tau < \epsilon$) around this adversarial example $\vx'$ instead of the original input $\vx$
\begin{equation}
	\L_{\text{\sabr}}\coloneqq \max_{\vx' \in \B(\vx', \epsilon - \tau)} \L_{\text{IBP}}(\vx', y, \tau).
\end{equation}
This generally yields a more precise (although not sound) worst-case loss approximation, thereby reducing over-regularization and improving both standard and certified accuracy.
% Note that while \sabr thus pursues the same goal as \tool (this work), its approach is orthogonal, as we will show later.

% This reduces approximation errors (see \cref{fig:empirical_bound_tightness}) and thus regularization, leading to improved certified and standard accuracy. \todo{Discuss difference?}
%

% \paragraph{Key Challenge: Worst Case Loss Approximation}
% Adversarial Training often does not find the worst-case perturbations, but rather a lower bound $\L_{\text{adv}} < \max_{\vx' \in \B(\vx, \epsilon)} \L(\vx', y)$, thus leading to insufficient regularization and networks with high empirical but small certifiable robustness. In contrast, certified training computes sound over-approximations of the worst-case loss but accumulates exponentially growing approximation errors in the process, leading to over-regularization and thus networks with low standard accuracy. 
% %
% While more precise propagation techniques \citep{ZhangCXGSLBH20,BalunovicV20,WongSMK18} reduce these over-approximation errors, they also lead to substantially harder optimization problems and as a result even lower accuracy. This counterintuitive phenomenon is called the paradox of certified training \citep{jovanovic2022paradox}. 
% %
% Thus the main challenge of adversarial training is to define a precise approximation of the true worst-case loss that yields a well-behaved optimization problem.

\begin{figure*}[t]
\centering

\resizebox{1.0\textwidth}{!}{
\begin{tikzpicture}
	\def\ysep{2.3cm};
	\def\xsep{0.5cm};
	\coordinate (arrowsepx) at (0.1, 0);
	\coordinate (arrowsepy) at (0.0, 0.1);

	% ============= Input Box =============
	\node[pane, minimum width=2.8cm, minimum height=3.4cm] (input_box) at (0.5, 1.0) {};
	\node[align=center, anchor=north] at (input_box.north) {\textbf{Input} $\vx$};
	\node[below right=\xsep and 0.0cm of input_box.north, anchor=north] (input_content) {
	\centering
	\begin{tikzpicture} 
		\draw[exactstyle] (-0.5, 0) -- (1.5, 0) -- (1.5, 2) -- (-0.5, 2) -- cycle;
		\draw[ibpstyle] (-0.5, 0) -- (1.5, 0) -- (1.5, 2) -- (-0.5, 2) -- cycle;
		\node[point] at (0.5,0.95) {};
		\node at (0.65,0.95) {\small $\vx$};
		\draw [line width=0.5pt, decoration={brace, mirror, raise=1mm}, decorate] (-0.5, 0) -- (1.5, 0);
		\node at (0.5,-0.2) {\small $2\epsilon$};
	\end{tikzpicture} 
	};

	% ============ Intermediate ===========
	\node[pane, minimum width=3.2cm, minimum height=3.4cm, below right=0 and \ysep of input_box.north east, anchor=north west] (intermediate_box) {};
	\node[align=center, anchor=north] at (intermediate_box.north) {\textbf{Embedding Space}};
	\node[below right=\xsep and 0.2cm of intermediate_box.north, anchor=north] (intermediate_content) {
	\centering
	\begin{tikzpicture} 
		\draw[exactlatentstyle] (3.0, -0.2) -- (4.9, -0.2) -- (4.9, 2.2) -- (3.0, 2.2) -- cycle;
		\draw[exactstyle] (3.2, 1.2) -- (4.3, -0.2) -- (4.5, -0.2) -- (4.5, 1.5) -- (3.65, 2.0) -- (3.4, 2.0) -- cycle;
		\draw[ibpstyle] (3.0, -0.2) -- (4.9, -0.2) -- (4.9, 2.2) -- (3.0, 2.2) -- cycle;

		\node at (3.2,1.2) {\tiny $\underline{z}_1$};
		\node at (4.7,1.2) {\tiny $\overline{z}_1$};
		\node at (3.95,2.2) {\tiny $\overline{z}_2$};
		\node at (3.95,0.2) {\tiny $\underline{z}_2$};

		\draw[-stealth, pgdarrow] (4.25, 1.1) -- (4.75, 1.6);
		\draw[-stealth, pgdarrow] (4.75, 1.6) -- (4.9, 2.1);
		\node[point, pgdarrow] () at (4.9, 2.1) {};
		\node[color=cpgd] at (5.15, 2.2) {\tiny $\hat{\vz}'$};

		\draw[-stealth, pgdarrow] (3.85, 1.0) -- (3.25, 0.5);
		\draw[-stealth, pgdarrow] (3.25, 0.5) -- (3.25, -0.2);
		\node[point, pgdarrow] at (3.25, -0.2) {};
		\node[color=cpgd] at (3.35, -0.16) {\tiny $\hat{\vz}''$};
	\end{tikzpicture} 
	};

	% ============ Output ===========
	\node[pane, minimum width=5.7cm, minimum height=3.1cm, below right=0 and \ysep of intermediate_box.north east, anchor=north west] (output_box) {};
	\node[align=center, anchor=north] at (output_box.north) {\textbf{Output}};

	% Large box   
	\begin{scope}[shift={(-5,0)}]
	\node[below right=\xsep and -0.13cm of output_box.north, anchor=north] (output_content) {
	\centering
	\begin{tikzpicture} 
		\draw[exactlatentstyle] (9.0, -0.1) -- (9.1, -0.6) -- (9.2, -1.1) -- (9.4, -1.5) -- (9.9, -1.35) -- (10.2, -1.1) -- (10.3, -0.9) -- (10.2, -0.4)  -- (10., -0.1) -- (9.9, 0.3) -- (9.5,0.1)--(9.3, 0.1);
		\draw[exactstyle] (9.4, -0.2) -- (9.3, -0.6) -- (9.3, -1.1) -- (9.4, -1.2) -- (9.9, -1.2) -- (10.1, -0.8) -- (10.0, -0.5) -- (10.0, -0.2) -- (9.6, -0.1);

		\draw[-stealth, pgdarrow] (9.65, -0.6) -- (9.5, -0.2);
		\draw[-stealth, pgdarrow] (9.5, -0.2) -- (9.5, 0.1);
		\node[point, pgdarrow] () at (9.5, 0.1) {};

		\draw[-stealth, pgdarrow] (9.8, -0.6) -- (10.1, -0.7);
		\draw[-stealth, pgdarrow] (10.1, -0.7) -- (10.3, -0.9);
		\node[point, pgdarrow] () at (10.3, -0.9) {};

		\draw[ibpstyle] (8.6, 0.6) -- (10.8, 0.6) -- (10.8, -1.5) -- (8.6, -1.5)-- cycle;
		\draw[toolstyle] (8.3, 0.1) -- (10.3, 0.1) -- (10.3, -1.8);

		\node[wcpoint] () at (9.9, 0.3) {};
		\node[wcpoint] () at (10.3, 0.1) {};
		\node[wcpoint] () at (10.8, 0.6) {};
		\node[wcpoint] () at (10.0, -0.2) {};

		% \draw[-{Triangle[width=4pt,length=2.5pt]}, line width=2pt](11.0, - 0.3) -- (11.4, -0.7);
		% \node[right, font=\footnotesize, text=black, anchor = west] at (11.2,-0.35) {Loss increases};

		% "legend"
		\draw[black] (10.8, 0.3) -- (11.1, 0.45) node[right, font=\footnotesize, text=black] {Box relaxation};
		\draw[black] (10.3, -.25) -- (11.1, 0.0) node[right, font=\footnotesize, text=black] {\tool};
		\draw[black] (10.22, -0.5) -- (11.1, -0.6) node[right, font=\footnotesize, text=black] {Exact propagation};
		\draw[black] (10., -0.35) -- (11.1, -0.6) node[right, font=\footnotesize, text=black] {Exact propagation};
	\end{tikzpicture} 
	};
\end{scope}

	% ============== NN 1 =============
	\node[rectangle, fill, align=center, color=netinside, text=black] (fe) at ($(input_box.east)!0.35!(intermediate_box.west)$) {$\vf_{E}$};
	\draw[-stealth, fwdarrow] ($(input_box.east)+(arrowsepx)$) -- ($(fe.west)-(arrowsepx)$);
	\draw[-stealth, fwdarrow] ($(fe.east)+(arrowsepx)$) -- ($(intermediate_box.west)-(arrowsepx)$);

	% ============== NN 2 =============  
	\coordinate (fcO) at  ({$(intermediate_box.east)$} -| {$(output_box.north west)$});
	\node[rectangle, fill, align=center, color=netinside, text=black] (fc) at ($(intermediate_box.east)!0.65!(fcO)$) {$\vf_{C}$};
	\draw[-stealth, fwdarrow] ($(intermediate_box.east)+(arrowsepx)$) -- ($(fc.west)-(arrowsepx)$);
	\draw[-stealth, fwdarrow] ($(fc.east)+(arrowsepx)$) -- ($(fcO)-(arrowsepx)$);

	% ============== Loss =============  
	\node [pane, below left = 0.55cm and 0cm of output_box.south, align=left, anchor=north, minimum height=0.1cm] (loss) {Training Loss $\mathcal{L}$};
	\draw[-stealth, fwdarrow] ($(output_box.south)-(arrowsepy)$) -- ($(loss.north)+(arrowsepy)$);

	% ============== Gradient Connector =============  
	\node [pane, align=center, anchor=center, minimum width=2.7cm, minimum height=0.1cm] (grad) at ({$(intermediate_box)$} |- {$(loss)$}) {\textbf{Gradient Connector}};

	% ============== Backward Pass =============  
	\draw[-stealth, bwdarrow] ($(loss.west)-(arrowsepx)$) -| ($(fc.south)-(arrowsepy)$) ;
	\node[below left = 0.1cm and 0.05cm of fc.south, color=cbwd] () {$\frac{d \mathcal{L}}{d \vtheta_C}$};
	\draw[-stealth, bwdarrow] ($(loss.west)-(arrowsepx)$) --  ($(grad.east)+(arrowsepx)$);
	\node[above right = 0.03cm and -0.05cm of grad.east, color=cbwd] () {$\frac{d \mathcal{L}}{d \hat{\vz}'}$, $\frac{d \mathcal{L}}{d \hat{\vz}''}$};
	\draw[-stealth, bwdarrow] ($(grad.west)-(arrowsepx)$)  -| ($(fe.south)-(arrowsepy)$);
	\node[below right = 0.1cm and 0.05cm of fe.south, color=cbwd] () {$\frac{d \mathcal{L}}{d \vtheta_E}$};
	\node[above left = 0.025cm and 0.05cm of grad.west, color=cbwd] () {$\frac{d \mathcal{L}}{d \overline{\vz}}$, $\frac{d \mathcal{L}}{d \underline{\vz}}$};
\end{tikzpicture}
}

\vspace{-1mm}
\caption{
Overview of \tool training.
First, forward propagation (\markerfwd) of a region $B(\vx, \epsilon)$ (\markerexact, left) around an input $\vx$ (\markerpoint) through the feature extractor $\vf_E$ yields the exact reachable set (\markerexact, middle) and its \ibp approximation $[\underline{\vz}, \overline{\vz}]$ (\markeribp, middle) in the embedding space. Further \ibp propagation through the classifier $\vf_C$ would yield an imprecise box approximation  (\markeribp, right) of the reachable set (\markerexact, right). Instead, \tool conducts an adversarial attack (\markerpgd) in the embedding space \ibp approximation (\markerexactlatent) yielding an under-approximation (\markertool) of its reachable set (\markerexactlatent, right).
We illustrate the points realizing the worst-case loss in every output region with \markerwc and enable back-propagation (\markerbwd) through the adversarial attack by introducing the gradient connector (discussed in \cref{sec:connector}).
}
\vspace{-2mm}
\label{fig:overview}
\end{figure*}
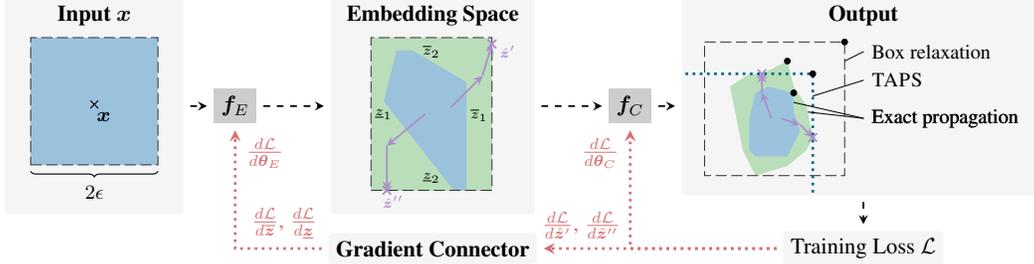

\section{Precise Worst-Case Loss Approximation} \label{sec:tool}
In this section, we first introduce \tool, a novel certified training method combining \ibp and \pgd training to obtain more precise worst-case loss estimates, before showing that this approach is orthogonal and complementary to current state-of-the-art methods.

\subsection{\tool~-- Combining \ibp and \pgd}
The key insight behind \tool is that adversarial training with PGD and certified training with IBP complement each other perfectly: (i) both yield well-behaved optimization problems, as witnessed by their empirical success, and (ii) we can combine them such that the over-approximation errors incurred during \ibp are compensated by the under-approximations of \pgd. 
\tool harnesses this as follows: For every sample, we first propagate the input region part-way through the network using \ibp and then conduct \pgd training within the thus obtained \boxd approximation. The key technical challenge with this approach %(and the main difference to COLT \cite{BalunovicV20}) 
lies in connecting the gradients of the two propagation methods and thus enabling joint training of the corresponding network portions. 

We now explain \tool in more detail along the illustration in \cref{fig:overview}. We first partition a neural network $\vf$ with weights $\vtheta$ into a \emph{feature extractor} $\vf_E$ and a \emph{classifier} $\vf_C$ with parameters $\vtheta_E$ and $\vtheta_C$, respectively, such that we have $\vf_\vtheta = \vf_C \circ \vf_E$ and $\vtheta = \vtheta_E \cup \vtheta_C$.
We refer to the output space of the feature extractor as the \emph{embedding space}. 
Given an input sample $\vx$ (illustrated as \markerpoint in \cref{fig:overview}) and a corresponding input region $\B(\vx, \epsilon)$ (\markerexact in the input panel), training proceeds as follows:
During the forward pass (black dashed arrows \markerfwd), we first use \ibp to compute a \boxd over-approximation $[\underline{\vz}, \overline{\vz}]$ (dashed box \markeribp) of the feature extractor's exact reachable set (blue region \markerexact), shown in the middle panel of \cref{fig:overview}. Then, we conduct separate adversarial attacks (\markerpgd) within this region in the embedding space (\markerexactlatent) to bound all output dimensions of the classifier. This yields latent adversarial examples $\hat{\vz} \in [\underline{\vz}, \overline{\vz}]$, defining the \tool bounds $\overline{\vo}_{\text{\tool}}^\Delta$ (dotted lines \markertool in the output space) on the network's output. This way, the \emph{under-approximation} of the classifier via \pgd, partially compensates the \emph{over-approximation} of the feature extractor via \ibp. Full \ibp propagation, in contrast, continues to exponentially accumulate approximation errors \citep{MuellerEFV22,ShiWZYH21}, yielding the much larger dashed box \markeribp. 
We now compute the \tool loss $\L_{\text{\tool}}$ analogously to $\L_{\text{IBP}}$ (\cref{eq:ibp}) by plugging the \tool bound estimate $\overline{\vo}_{\text{\tool}}^\Delta$ into the Cross-Entropy loss.
Comparing the resulting losses (illustrated as \markerwc and growing towards the top right), we see that while the \tool bounds are not necessarily sound, they yield a much better approximation of the true worst-case loss.

%Note that, as we will discuss in \cref{sec:objective}, the overall loss $\L$ includes terms other than $\L_{\text{\tool}}$. 

During the backward pass (orange dotted arrows \markerbwd in \cref{fig:overview}), we compute the gradients w.r.t. the classifier's parameters $\vtheta_C$ and the latent adversarial examples $\hat{\vz}$ (classifier input) as usual. However, to compute the gradients w.r.t. the feature extractor's parameters $\vtheta_F$, we have to compute (pseudo) gradients of the latent adversarial examples $\hat{\vz}$ w.r.t. the box bounds $\underline{\vz}$ and $\overline{\vz}$. As these gradients are not well defined, we introduce the \emph{gradient connector}, discussed next, as an interface between the feature extractor and classifier, imposing such pseudo gradients.
This allows us to train $\vf_E$ and $\vf_C$ \emph{jointly}, leading to a feature extractor that minimizes approximation errors and a classifier that is resilient to the spurious points included in the remaining approximation errors.
%This allows us to \emph{jointly} train the $\vf_E$ and $\vf_C$ and thus, in contrast to \colt \citep{BalunovicV20} (further discussed in \cref{sec:related_work}), not only learn to handle approximation errors but also minimize them.

\subsection{Gradient Connector} \label{sec:connector}
The key function of the gradient connector is to enable gradient computation through the adversarial example search in the embedding space. Using the chain rule, this only requires us to define the (pseudo) gradients $\frac{\partial \hat{\vz}}{\partial \underline{\vz}}$ and $\frac{\partial \hat{\vz}}{\partial \overline{\vz}}$ of the latent adversarial examples $\hat{\vz}$ w.r.t. the box bounds  $\underline{\vz}$ and $\overline{\vz}$ on the feature extractor's outputs. Below, we will focus on the $i$\th dimension of the lower box bound $\underline{z}_i$ and note that all other dimensions and the upper bounds follow analogously.
%, depending on the box bounds $[\underline{\vz}, \overline{\vz}]$ and the concrete latent adversarial example $\hat{\vz}$.
%During the forward pass, it simply saves these variables for a later gradient computation while otherwise acting like an identity function.
%
% We now discuss how we compute these pseudo gradients for the lower bounds and note that the upper bounds follow analogously. Recall that using the chain rule we obtain the desired gradients as
% \begin{equation*}
%     \frac{\partial \hat{L}}{\partial \underline{\vz}} = \frac{\partial \hat{L}}{\partial \hat{\vz}} \frac{\partial \hat{\vz}}{\partial \underline{\vz}}.
% \end{equation*}
%Considering just the $i$\th dimension of the box bounds, we obtain $\frac{\partial \hat{L}}{\partial \underline{z}_i} = \sum_j \frac{\partial \hat{L}}{\partial \hat{z}_j} \frac{\partial \hat{z}_j}{\partial \underline{\vz}_i}$. 

As the latent adversarial examples can be seen as multivariate functions in the box bounds, we obtain the general form $\frac{d \L}{d \underline{z}_i} = \sum_j \frac{d \L}{d \hat{z}_j} \frac{\partial \hat{z}_j}{\partial \underline{z}_i}$, depending on all dimensions of the latent adversarial example. 
We now consider a single \pgd step and observe that bounds in the $i$\th dimension have no impact on the $j$\th coordinate of the resulting adversarial example as they impact neither the gradient sign nor the projection in this dimension, as \boxd bounds are axis parallel.
We thus assume independence of the $j$\th dimension of the latent adversarial example $\hat{z}_j$ from the bounds in the $i$\th dimension $\underline{z}_i$ and $\overline{z}_i$ (for $i \neq j$), which holds rigorously (up to initialization) for a single step attack and constitutes a mild assumption for multi-step attacks.
Therefore, we have $\frac{\partial \hat{z}_j}{\partial \underline{z}_i} = 0$ for $i \ne j$ and obtain $\frac{d \L}{d \underline{z}_i} =\frac{d \L}{d \hat{z}_i} \frac{\partial \hat{z}_i}{\partial \underline{z}_i}$, leaving only $\frac{\partial \hat{z}_i}{\partial \underline{z}_i}$ for us to define.

The most natural gradient connector is the \emph{binary connector}, \ie, set $\frac{\partial \hat{z}_i}{\partial \underline{z}_i} = 1$ when $\hat{z}_i = \underline{z}_i$ and $0$ otherwise, as it is a valid sub-gradient for the projection operation in \pgd. However, the latent adversarial input often does not lie on a corner (extremal vertex) of the \boxd approximation, leading to sparse gradients and thus a less well-behaved optimization problem. 
More importantly, the binary connector is very sensitive to the distance between (local) loss extrema and the box boundary and thus inherently ill-suited to gradient-based optimization. For example, a local extremum at $\hat{z}_i$ would induce $\frac{\partial \hat{z}_i}{\partial \underline{z}_i} = 1$ in the box $[\hat{z}_i, 0]$, but $\frac{\partial \hat{z}_i}{\partial \underline{z}_i} = 0$ for $[\hat{z}_i-\epsilon, 0]$, even for arbitrarily small $\epsilon$.

\begin{wrapfigure}[9]{r}{0.43 \textwidth}
    \centering
    \vspace{4mm}
    \hspace{-12mm}
    % \includegraphics[width=.6\linewidth]{figures/softlinear_sketch.pdf}
    % \vspace{-3mm}
    % \begin{minipage}{.870\linewidth}
        % \centering
        % \scalebox{0.9}{
    \begin{tikzpicture}
    \def\x{4.5}
    \def\y{1.5}
    \def\c{0.3}
    \def\dy{0.075}
    \def\dx{0.1}
    \pgfmathsetmacro\cc{1-\c}

    \node[pane, minimum width=5.7cm, minimum height=3.0cm] (input_box) at (0.6*\x, 0.68*\y) {};
    \node[anchor = south, align=center] at (input_box.south) {Coordinate $\hat{z}$};
    
    \centering
    \begin{tikzpicture} 
        \draw[-, line width = 0.5pt] (0, 0) -- (\x, 0);
        \draw[-, line width = 0.5pt] (0, 0) -- (0, \y);

        \draw[-] ($(0,-\dx)$) -- ($(0,\dx)$);
        \node[anchor=north, align=center, scale=0.85] at ($(0,-\dx-0.05)$) {$\underline{z}$};
        \draw[-] ($(\c*\x,-\dx)$) -- ($(\c*\x,\dx)$);
        \node[anchor=north, align=center, scale=0.85] at ($(\c*\x,-\dx)$) {$\underline{z} + c (\overline{z} - \underline{z})$};
        \draw[-] ($(\cc*\x,-\dx)$) -- ($(\cc*\x,\dx)$);
        \node[anchor=north, align=center, scale=0.85] at ($(\cc*\x,-\dx)$) {$\overline{z} - c (\overline{z} - \underline{z})$};
        \draw[-] ($(\x,-\dx)$) -- ($(\x,\dx)$);
        \node[anchor=north, align=center, scale=0.85] at ($(\x,-\dx-0.05)$) {$\overline{z}$};

        \draw[-] ($(-\dy,0*\y)$) -- ($(\dy,0*\y)$);
        \node[anchor=east, align=center, scale=0.85] at ($(-\dy, 0*\y)$) {$0.0$};
        % \draw[-] ($(-\dy,0.25*\y)$) -- ($(\dy,0.25*\y)$);
        % \node[anchor=east, align=center, scale=0.65] at ($(-\dy, 0.25*\y)$) {$0.25$};
        \draw[-] ($(-\dy,0.5*\y)$) -- ($(\dy,0.5*\y)$);
        \node[anchor=east, align=center, scale=0.85] at ($(-\dy, 0.5*\y)$) {$0.5$};
        % \draw[-] ($(-\dy,0.75*\y)$) -- ($(\dy,0.75*\y)$);
        % \node[anchor=east, align=center, scale=0.65] at ($(-\dy, 0.75*\y)$) {$0.75$};
        \draw[-] ($(-\dy,1*\y)$) -- ($(\dy,1*\y)$);
        \node[anchor=east, align=center, scale=0.85] at ($(-\dy, 1*\y)$) {$1.0$};

        \draw[-,my-full-blue!90, line width = 1.5pt, opacity=0.65] (0., \y) -- node[pos=0.5, above right=-0.1cm and -0.1cm, scale=1.3, black] {$\frac{\partial \hat{z}}{\partial \overline{z}}$}  ($(\c*\x, 0)$) -- (\x, 0);
        \draw[-,my-full-orange!90, line width = 1.5pt, opacity=0.65] (0, 0) -- ($(\cc*\x, 0)$) -- node[pos=0.5,  above left=-0.17cm and -0.1cm, scale=1.3, black] {$\frac{\partial \hat{z}}{\partial \underline{z}}$} (\x, \y) ;
        
        % \draw[-,my-full-blue!70, line width = 1.5pt] (1.1 * \x, 0.65 * \y) -- ($(1.1 * \x + 0.3, 0.65 * \y)$) node[right, scale=0.9, black] {$\frac{\partial \hat{z}}{\partial \overline{z}}$};
        % \draw[-,my-full-orange!70, line width = 1.5pt] (1.1 * \x, 0.3 * \y) -- ($(1.1 * \x + 0.3, 0.3 * \y)$) node[right, scale=0.9, black] {$\frac{\partial \hat{z}}{\partial \underline{z}}$};
    \end{tikzpicture}

\end{tikzpicture}    
        % }
    % \end{minipage}
    \vspace{-3mm}
    \caption{Gradient connector visualization.} 
    \label{softlinear}
\end{wrapfigure}
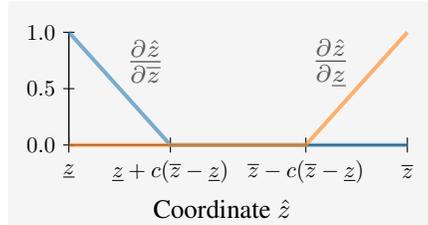
To alleviate both of these problems, we consider a \emph{linear connector}, \ie, set $\frac{\partial \hat{z}_i}{\partial \underline{z}_i} = \frac{\overline{z}_i - \hat{z}_i}{\overline{z}_i - \underline{z}_i}$. However, even when our latent adversarial example is very close to one bound, the linear connector would induce non-zero gradients w.r.t. to the opposite bound. To remedy this undesirable behavior, we propose the \emph{rectified linear connector}, setting $\frac{\partial \hat{z}_i}{\partial \underline{z}_i} = \max(0, 1 - \frac{\hat{z}_i - \underline{z}_i}{ c(\overline{z}_i - \underline{z}_i)})$ where $c \in [0, 1]$ is a constant (visualized in \cref{softlinear} for $c=0.3$).
Observe that it recovers the binary connector for $c = 0$ and the linear connector for $c=1$. To prevent gradient sparsity ($c \leq 0.5$) while avoiding the above-mentioned counterintuitive gradient connections ($c \geq 0.5$), we set $c=0.5$ unless indicated otherwise. %This ensures that, for every dimension, exactly one of the bounds has a non-zero gradient, unless our latent adversarial input is centered between them, \ie, $\hat{z}_i = (\underline{z}_i + \overline{z}_i)/2$, where both have zero gradients.
When the upper and lower bounds are identical in the $i$\th dimension, \pgd turns into an identity function. Therefore, we set both gradients to $\frac{\partial \hat{z}_i}{\partial \underline{z}_i}=\frac{\partial \hat{z}_i}{\partial \overline{z}_i} = 0.5$ turning the gradient connector into an identity function for the backward pass.
%To avoid numerical instability when the upper and lower bounds are identical in the $i$\th dimension, \ie, $\underline{z}_i = \overline{z}_i$, we set $\frac{\partial \hat{z}_i}{\partial \underline{z}_i} = \max(0, 1 - \frac{\max(\hat{z}_i - \underline{z}_i, \delta/2)}{\max(c(\overline{z}_i - \underline{z}_i), \delta)})$, where $\delta=10^{-5}$ is a small positive constant. Observe that this clipping is negligible unless the lower and the upper bounds are very close while yielding $\frac{\partial \hat{z}_i}{\partial \underline{z}_i}=\frac{\partial \hat{z}_i}{\partial \overline{z}_i} = 0.5$ when they are identical, thus turning the gradient connector into an identity function.

% As outlined before, our approach is conceptually similar to \colt \cite{BalunovicV20}. However, our construction and gradient connector allow us to overcome the main issue of \colt, \ie, that gradients do not flow into the earlier part of the neural network -- in our case $\vf_E$.

\begin{wrapfigure}[15]{r}{0.36 \textwidth}
    \centering
    \vspace{-5mm}
    \scalebox{1.10}{
    \begin{tikzpicture}
    
        % ============ Output ===========
        \node[pane, minimum width=3.74cm, minimum height=2.65cm, anchor=north west] (output_box) {};
        % \node[align=center, anchor=north] at (output_box.north) {\textbf{Output}};
    
        % Large box   
        \node[below right=-0.1cm and -0.15cm of output_box.north, anchor=north] (output_content) {
        \centering
        \begin{tikzpicture} 
            \draw[-stealth, line width = 0.5pt] (-2.2, 1.2) -- (-0.1, 1.2) node[right,scale=0.85] {$o^\Delta_1$};
            \draw[-stealth, line width = 0.5pt] (-0.7, -0.1) -- (-.7, 1.8) node[above,scale=0.85] {$o^\Delta_2$};

            \draw[exactlatentstyle] (-2.0, 0.1) -- (-1.9, 0.6) -- (-1.8, 1.1) -- (-1.6, 1.6) -- (-1.1, 1.35) -- (-0.8, 1.1) -- (-0.5, 0.5) -- (-0.6, 0.4)  -- (-0.8, 0.1) -- (-1.2, -0.2) -- (-1.5,-0.1)--(-1.7, -0.1);
            % \draw[exactstyle] (-1.6, 0.2) -- (-1.7, 0.6) -- (-1.7, 1.1) -- (-1.6, 1.2) -- (-1.1, 1.2) -- (-0.9, 0.8) -- (-1.0, 0.5) -- (-1.0, 0.2) -- (-1.4, 0.1);
    
            \draw[-stealth, pgdarrow] (-1.3, 0.1) -- (-0.9, 0.3);
            \draw[-stealth, pgdarrow] (-0.9, 0.3) -- (-0.5, 0.5);
            \node[point, pgdarrow] () at (-0.5, 0.5) {};
    
            \draw[-stealth, pgdarrow] (-1.4, 0.9) -- (-1.7, 1.2);
            \draw[-stealth, pgdarrow] (-1.7, 1.2) -- (-1.6, 1.6);
            \node[point, pgdarrow] () at (-1.6, 1.6) {};

            \draw[-stealth, pgdsinglearrow] (-1.4, 0.6) -- (-1.0, 0.7);
            \draw[-stealth, pgdsinglearrow] (-1.0, 0.7) -- (-0.8, 1.1);
            \node[point, pgdsinglearrow] () at (-0.8, 1.1) {};
    
            \draw[pgdarrow,  dotted, line width =1.1] (-0.5, -0.3) -- (-0.5, 1.6) -- (-2.43, 1.6);
            \draw[pgdsinglearrow, dashed] (-0.8, -0.3) -- (-0.8, 1.1) -- (-2.43, 1.1);

            % \draw[toolstyle] (-2.0, -0.3) -- (-0.5, -0.3) -- (-0.5, 1.6) -- (-2.0, 1.6)-- cycle;
    
            \node[wcpoint] () at (-0.5, 1.6) {};
            % \draw[-Triangle, black] (-0.5, 1.6) -- (0.0, 2.1) node[below right=0.1cm and -0.2cm,scale=0.85] {$\nabla_{\vo^\Delta}\bc{L}_{\text{CE}}$};

            % \node[wcpoint] () at (0.3, -0.1) {};
    
            % \draw[-{Triangle[width=4pt,length=2.5pt]}, line width=2pt](1.0, - 0.3) -- (1.4, -0.7);
            % \node[right, font=\footnotesize, text=black, anchor = west] at (1.2,-0.35) {Loss increases};

            % "legend"
            \draw[black] (-0.5, 1.6) -- (0.0, 1.9) node[right, font=\footnotesize, text=black] {$\bc{L}_{\tool}$};
            \draw[black] (-0.8, 1.1) -- (0.0, 0.6) node[right, font=\footnotesize, text=black] {$\bc{L}^{\text{single}}_{\tool}$};
            % \draw[black] (0.22, 0.5) -- (1.1, 0.6) node[right, font=\footnotesize, text=black] {Exact propagation};
            % \draw[black] (0., 0.35) -- (1.1, 0.6) node[right, font=\footnotesize, text=black] {Exact propagation};
        \end{tikzpicture} 
        };

    \end{tikzpicture}
    }
    \vspace{-2mm}
    \caption{Illustration of the bounds on $o_i^\Delta \coloneqq o_i - o_t$ obtained via single estimator (\markersinglepgd) and multi-estimator (\markermultipgd) \pgd and the points maximizing the corresponding losses: \markerwcsingle for $\bc{L}^{\text{single}}_{\tool}$ and \markerwc for $\bc{L}_{\tool}$.
    }
    \label{fig:multi_estimator}
    % \vspace{-4mm}
    \end{wrapfigure}

\subsection{\tool Loss \& Multi-estimator \pgd} \label{sec:multi-estimator}
% TAPS equipped with the proposed rectified linear gradient connector works well for networks without batch norm (BN). However, when there are batch norm layers in the network, the PGD propagation becomes problematic. Specifically, \citet{XieTGWYL20} found that the activations of the adversarial inputs follow a different distribution to the activations of the original input. However, currently, the best practice for IBP is to use the original input to set the batch statistics \cite{ShiWZYH21}. This problem leads to performance degradation when PGD training is applied on BN networks. \mm{Not sure how much we want to address BN here? I think if we highlight it so prominently, we definitely need to run some ablations analyzing this.}
The standard \pgd attack, used in adversarial training, henceforth called \emph{single-estimator} \pgd, is based on maximizing the Cross-Entropy loss $\L_{\text{CE}}$ of a single input. In the context of \tool, this results in the overall loss
\begin{equation*}
    % \resizebox{0.9\hsize}{!}{$\displaystyle
	\bc{L}^{\text{single}}_{\tool}(\vx, y, \epsilon) = \max_{\hat{\vz} \in [\underline{\vz}, \overline{\vz}]} \ln \Big( 1 + \sum_{i \neq y} \exp(f_C(\hat{\vz})_i - f_C(\hat{\vz})_y) \Big),
    % $}
\end{equation*}
where the embedding space bounding box $[\underline{\vz}, \overline{\vz}]$ is obtained via \ibp.
However, this loss is not necessarily well aligned with adversarial robustness. Consider the example illustrated in \cref{fig:multi_estimator}, where only points in the lower-left quadrant are classified correctly (\ie, $o_i^\Delta \coloneqq o_i - o_y < 0$). We compute the latent adversarial example $\hat{\vz}$ by conducting a standard adversarial attack on the Cross-Entropy loss over the reachable set \markerexactlatent (optimally for illustration purposes) and observe that the corresponding output $\vf(\hat{\vz})$ (\markerwcsingle) is classified correctly. However, if we instead use the logit differences $o_1^\Delta$ and $o_2^\Delta$ as attack objectives, we obtain two misclassified points (\markeradvpoint). Combining their dimension-wise worst-case bounds (\markermultipgd), we obtain the point \markerwc, which realizes the maximum loss over an optimal box approximation of the reachable set. 
As the correct classification of this point (when computed exactly) directly corresponds to true robustness, we propose the \emph{multi-estimator} \pgd variant of $\bc{L}_{\tool}$, which estimates the upper bounds on the logit differences $o_i^\Delta$ using separate samples and then computes the loss function using the per-dimension worst-cases as:
\begin{equation*}
    % \resizebox{0.9\hsize}{!}{$\displaystyle
    \bc{L}_{\tool}(\vx, y, \epsilon) = \ln \Big( 1 + \sum_{i \neq y} \exp\Big( \max_{\hat{\vz} \in [\underline{\vz}, \overline{\vz}]} f_C(\hat{\vz})_i - f_C(\hat{\vz})_y \Big) \Big).
    %$}
\end{equation*}

\subsection{Training Objective \& Regularization} \label{sec:objective}
While complete certification methods can decide any robustness property, this requires exponential time. Therefore, networks should not only be robust but also certifiable.
Thus, we propose to combine the IBP loss for easy-to-learn and certify samples with the \tool loss for harder samples as follows: 
\begin{equation*}
    \bc{L}(\vx, y, \epsilon) = \bc{L}_\text{\tool}(\vx, y, \epsilon) \cdot \bc{L}_\text{\ibp}(\vx,y, \epsilon).
\end{equation*}
This expresses that every sample should be either certifiable with \tool or \ibp bounds\footnote{See \citet{FischerBDGZV19} for further discussion.}.
Further, as by construction $\bc{L}_\text{\tool} \leq \bc{L}_\text{\ibp}$, we add a scaling term $\alpha$ to the loss gradient: 
\begin{equation*}
	\frac{d\L}{d \vtheta} \coloneqq 
	2\alpha \frac{d\L_\text{\tool}}{d \vtheta} \cdot \L_\text{\ibp}+
	(2-2\alpha) \frac{d\L_\text{\ibp}}{d \vtheta} \cdot \L_\text{\tool}.
\end{equation*}
Here, $\alpha = 0.5$ recovers the standard gradient, obtained via the product rule (both sides weighted with $1$), while $\alpha = 0$ and $\alpha = 1$ correspond to using only the (weighted) \ibp and \tool gradients, respectively. 
Henceforth, we express this as the regularization weight $w_\text{\tool} = \frac{\alpha}{1-\alpha}$, which intuitively expresses the weight put on \tool, using $w_\text{\tool} = 5$ unless specified otherwise.
Lastly, we reduce the variance of $\L$ by averaging $\L_\text{\ibp}$ and $\L_\text{\tool}$ over a mini batch before multiplying (see \cref{app:proof_ibp_reg}).

\subsection{\toolp ~-- Balancing Regularization by Combining \tool with \sabr} \label{sec:STAPS}
Recall that \sabr \citep{MuellerEFV22} reduces the over-regularization of certified training by propagating a small, adversarially selected \boxd through the whole network. However, as \boxd approximations grow exponentially with depth \citep{MuellerEFV22,ShiWZYH21,MaoMFV23}, regardless of the input region size, \sabr has to strike a balance between regularizing early layers too little and later layers too much.
In contrast, \tool's approach of propagating the full input region through the first part of the network (the feature extractor) before using \pgd for the remainder reduces regularization only in later layers. Thus, we propose \toolp by replacing the \ibp components of \tool, for both propagation and regularization, with \sabr to obtain a more uniform reduction of over-regularization throughout the whole network.

\toolp, identically to \sabr, first conducts a PGD attack over the whole network to find an adversarial example $\vx' \in \B(\vx', \epsilon - \tau)$.
Then, it propagates \boxd-bounds for a small region $\B(\vx', \tau) \subset \B(\vx, \epsilon)$ (with $\tau < \epsilon$) around this adversarial example $\vx'$ through the feature extractor, before, identically to \tool, conducting an adversarial attack in the resulting latent-space region over the classifier component of the network.

\section{Experimental Evaluation} \label{sec:eval}
In this section, we evaluate \tool empirically, first, comparing it to a range of state-of-the-art certified training methods, before conducting an extensive ablation study validating our design choices.% to analyze it more in-depth.

\paragraph{Experimental Setup}
We implement \tool in PyTorch \citep{PaszkeGMLBCKLGA19} and use \mnbab \citep{FerrariMJV22} for certification. %a cascade of increasingly precise verification methods for certification from IBP \citep{GehrMDTCV18} to \crownibp \citep{ZhangCXGSLBH20} and \mnbab \citep{FerrariMJV22}. 
We conduct experiments on \mnist \citep{lecun2010mnist}, \cifar \citep{krizhevsky2009learning}, and \TIN \citep{Le2015TinyIV} using $\ell_\infty$ perturbations and the \cnns architecture \citep{GowalIBP2018}. For more experimental details including hyperparameters and computational costs and an extended analysis see \cref{sec:exp_setting} and \cref{app:extended_exp}, respectively.

%use experiments to discuss the training difficulty, the effect of IBP regularization proposed in \cref{sec:objective}, the effect of the classifier split, the effect of the IBP regularization, and the combination with small box propagation \cite{MuellerEFV22}, another method proposed to reduce the regularization introduced by IBP.
%We use three certification methods sequentially, IBP $<$ CROWN-IBP \cite{ZhangCXGSLBH20} $<$ MN-BaB \cite{FerrariMJV22} in terms of both precision and time cost. By default, we use a CNN-7-BN as the model architecture, and split it into different subnetworks in the experiments.

% We report \emph{natural} accuracy, the portion of inputs that are correctly classified, \emph{adversarial} and \emph{certified} accuracy, representing an upper and lower bound, respectively, on the portion of inputs that are robustly correct, and \emph{\tool} accuracy, the portion of inputs where we don't find a latent adversarial example with \tool, corresponding to the goodness of fit (GoF).

\begin{table*}[t]
	\centering
	% \vspace{-2mm}
    \caption{Comparison of natural (Nat.) and certified (Cert.) accuracy on the full \mnist, \cifar, and \TIN test sets. We report results for other methods from the relevant literature.}
    \vspace{-1mm}
	\renewcommand{\arraystretch}{1.0}
	\begin{adjustbox}{width=0.75\linewidth,center}
		\begin{threeparttable}
			\begin{tabular}{@{}lclccc@{}}
				\toprule
				Dataset & $\epsilon_\infty$ & Training Method & Source & Nat. [\%]  & Cert. [\%] \\
				\midrule
				\multirow{12}*{\mnist}&\multirow{6}*{0.1}&\colt &\citet{BalunovicV20}     & 99.2  & 97.1 \\
				% &&\crownibp &\citet{ZhangCXGSLBH20}& 98.83 &  97.76 \\
				&&\ibp &\citet{ShiWZYH21}  & 98.84 & 97.95 \\
				&&\sortnet &\citet{ZhangJHW22}  & 99.01 & 98.14 \\
				&&\sabr &\citet{MuellerEFV22}            & \textbf{99.23} & 98.22  \\
				&&\tool  & this work & 99.19 & \textbf{98.39} \\
				&&\toolp  & this work & 99.15 & 98.37 \\

				\cmidrule(rl){2-6}
				&\multirow{6}*{0.3}                  &\colt &\citet{BalunovicV20}        & 97.3  & 85.7\\
				% &&\crownibp &\citet{ZhangCXGSLBH20}  & 98.18 & 92.98 \\
				&&\ibp &\citet{ShiWZYH21}    & 97.67 & 93.10 \\ 
				&&\sortnet &\citet{ZhangJHW22}  & 98.46 & 93.40 \\
				&&\sabr &\citet{MuellerEFV22}              & \textbf{98.75} & 93.40 \\
				&&\tool  & this work & 97.94 & \textbf{93.62} \\
				&&\toolp  & this work & 98.53 & 93.51 \\

				\cmidrule(rl){1-6}
				\multirow{16}*{\cifar} &\multirow{7}*{$\displaystyle \frac{2}{255}$} &\colt &\citet{BalunovicV20}                  & 78.4 & 60.5 \\
				% &&\crownibp &\citet{ZhangCXGSLBH20}            & 71.52 & 53.97 \\
				&&\ibp &\citet{ShiWZYH21}              & 66.84 & 52.85 \\
				&&\sortnet &\citet{ZhangJHW22}  & 67.72 & 56.94 \\
				&&\ibpr&\citet{PalmaIBPR22}  & 78.19 & 61.97\\
				&&\sabr &\citet{MuellerEFV22}         & 79.24 & 62.84\\
				&&\tool  & this work & 75.09 & 61.56 \\
				&&\toolp  & this work & \textbf{79.76} & \textbf{62.98} \\

				\cmidrule(rl){2-6}
				&\multirow{7}*{$\displaystyle \frac{8}{255}$}&\colt  &\citet{BalunovicV20}                        & 51.7 & 27.5 \\
				% &&\crownibp &\citet{XuS0WCHKLH20}                    & 46.29 & 33.38\\
				%               &&CROWN-IBP \citet{ZhangCXGSLBH20}\tnote{$\dagger$} && \textbf{54.50} && 30.50\\
				&&\ibp &\citet{ShiWZYH21}                     & 48.94 & 34.97 \\ 
				&&\sortnet &\citet{ZhangJHW22}  & \textbf{54.84} & \textbf{40.39} \\
				&&\ibpr &\citet{PalmaIBPR22}      &51.43 & 27.87\\
				&&\sabr &\citet{MuellerEFV22}                             & 52.38 & 35.13\\ 
				&&\tool  & this work & 49.76 & 35.10 \\
				&&\toolp  & this work & 52.82 & 34.65 \\

				\cmidrule(rl){1-6}
				\multirow{5}*{\TIN}&\multirow{5}*{$\displaystyle \frac{1}{255}$}& %\crownibp &\citet{ShiWZYH21}  & 25.62 & 17.93\\
				\ibp &\citet{ShiWZYH21} & 25.92 &17.87 \\
				&&\sortnet &\citet{ZhangJHW22}  & 25.69 & 18.18 \\
				&& \sabr &\citet{MuellerEFV22} & 28.85 & 20.46 \\
				&&\tool  & this work & 28.34 & 20.82 \\
				&&\toolp  & this work & \textbf{28.98} & \textbf{22.16} \\
				\bottomrule
			\end{tabular}
			%			\begin{tablenotes}
			%				\footnotesize
			%				%    \item[$\parallel$] Mean over three repetitions
			%%				\item[$\ast$] With shrinking
			%				%    \item[$\dagger$] Results for different hyperparameters
			%			\end{tablenotes}
			\label{tab:results}
		\end{threeparttable}
	\end{adjustbox}
	\vspace{-3mm}
\end{table*}

\subsection{Main Results}\label{sec:main_result}
In \cref{tab:results}, we compare \tool to state-of-the-art certified training methods. Most closely related are \ibp, recovered by \tool if the classifier size is zero, and \colt, which also combines bound propagation with adversarial attacks but does not allow for joint training.
\tool dominates \ibp, improving on its certified and natural accuracy in all settings and demonstrating the importance of avoiding over-regularization.
Compared to \colt, \tool improves certified accuracies significantly, highlighting the importance of joint optimization. In some settings, this comes at the cost of slightly reduced natural accuracy, potentially due to \colt's use of the more precise \zono approximations.
Compared to the recent \sabr and \ibpr, \tool often achieves higher certified accuracies at the cost of slightly reduced natural accuracies. Reducing regularization more uniformly with \toolp achieves higher certified accuracies in almost all settings and better natural accuracies in many, further highlighting the orthogonality of \tool and \sabr. Most notably, \toolp increases certified accuracy on \TIN by almost $10\%$ while also improving natural accuracy.
\sortnet, a generalization of a range of recent architectures \citep{ZhangCLHW21,ZhangJH022,AnilLG19}, introducing novel activation functions tailored to yield networks with high $\ell_\infty$-robustness, performs well on \cifar at $\epsilon = 8/255$, but is dominated by \toolp in every other setting.%, especially at smaller perturbation magnitudes.%. While \sortnet remains competitive on \mnist, \toolp reaches up to $17\%$ higher certified and natural accuracies on the more challenging \cifar and \TIN.

\subsection{Ablation Study}\label{sec:ablation}

\begin{wrapfigure}[20]{r}{0.48 \textwidth}
    \vspace{-4.5mm}
    \centering
    % \begin{subfigure}[b]{0.25\linewidth}
    %     \centering
    %     \includegraphics[width=1.03\linewidth]{figures/exact_tightness_IBP.pdf}
    %     \vspace{-6mm}
    %     \caption{IBP-trained}
    % \end{subfigure}
    % \hfill
    % \begin{subfigure}[b]{0.25\linewidth}
    %     \centering
    %     \includegraphics[width=1.03\linewidth]{figures/exact_tightness_SABR.pdf}
    %     \vspace{-6mm}
    %     \caption{\sabr-trained}
    % \end{subfigure}
    % \hfill
    % \begin{subfigure}[b]{0.42\linewidth}
        \centering
        \includegraphics[width=0.93\linewidth]{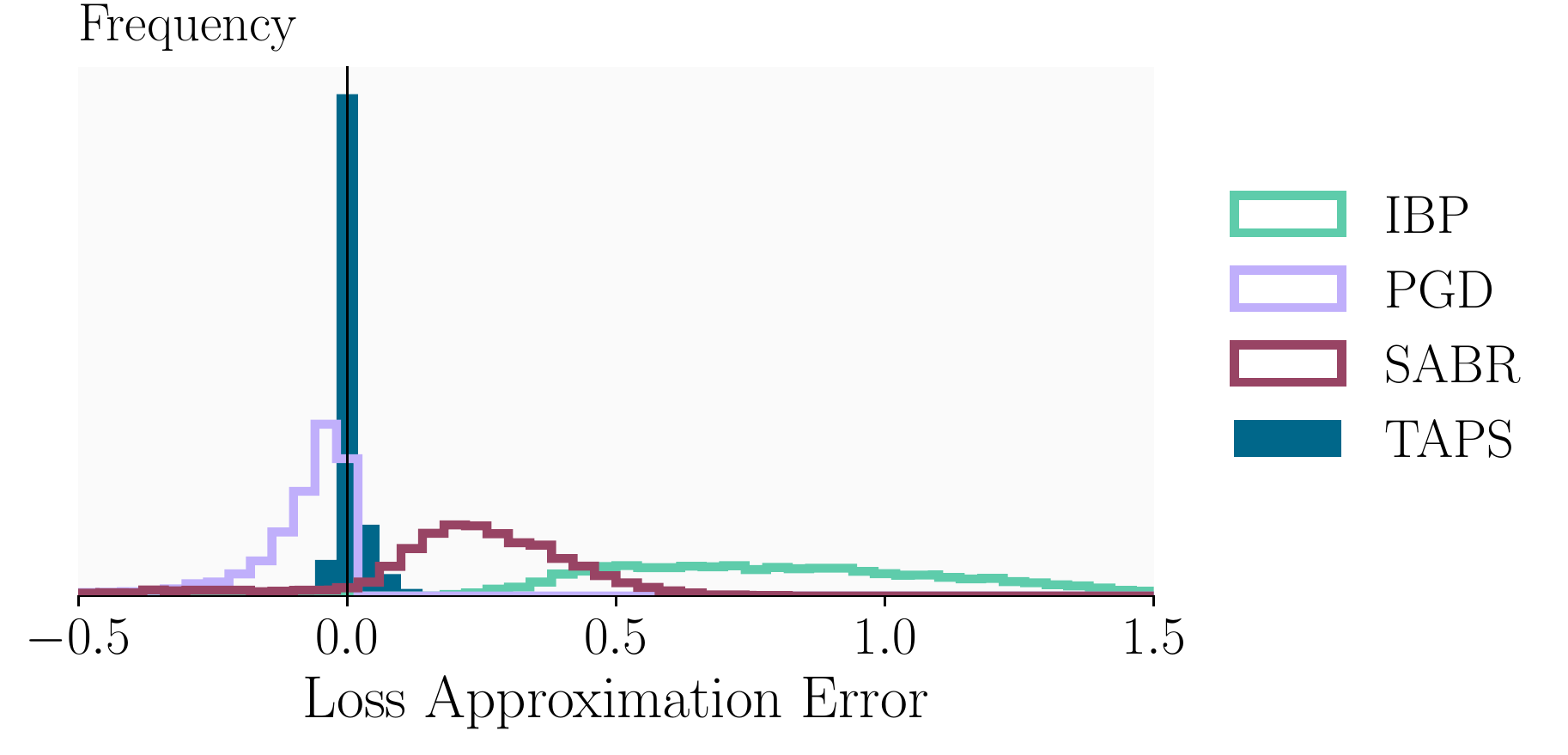}
        % \vspace{-6mm}
        % \caption{\tool-trained$\qquad\quad\;\;\;\;\;\;$}
    % \end{subfigure}
    \vspace{-2mm}
    \caption{Distribution of the worst-case loss approximation errors over test set samples.% with \tool training. Positive values correspond to over-approximations and negative values to under-approximations. We use an exact MILP encoding \citep{TjengXT19} as reference.
    }
    \label{fig:empirical_bound_tightness}
    \vspace{1mm}
    \begin{subfigure}[b]{0.41\linewidth}
        \centering
    \includegraphics[width=1.05\linewidth]{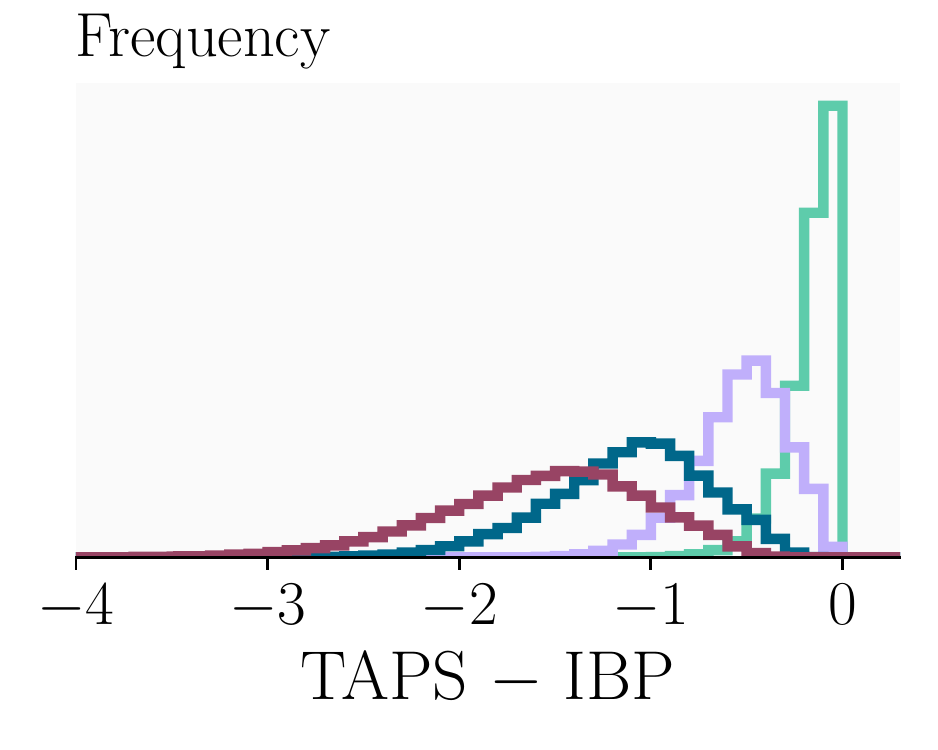}
    \vspace{-6mm}
    \caption{\tool}
    \end{subfigure}
    \hfill
    \begin{subfigure}[b]{0.574\linewidth}
        \centering
        \includegraphics[width=1.05\linewidth]{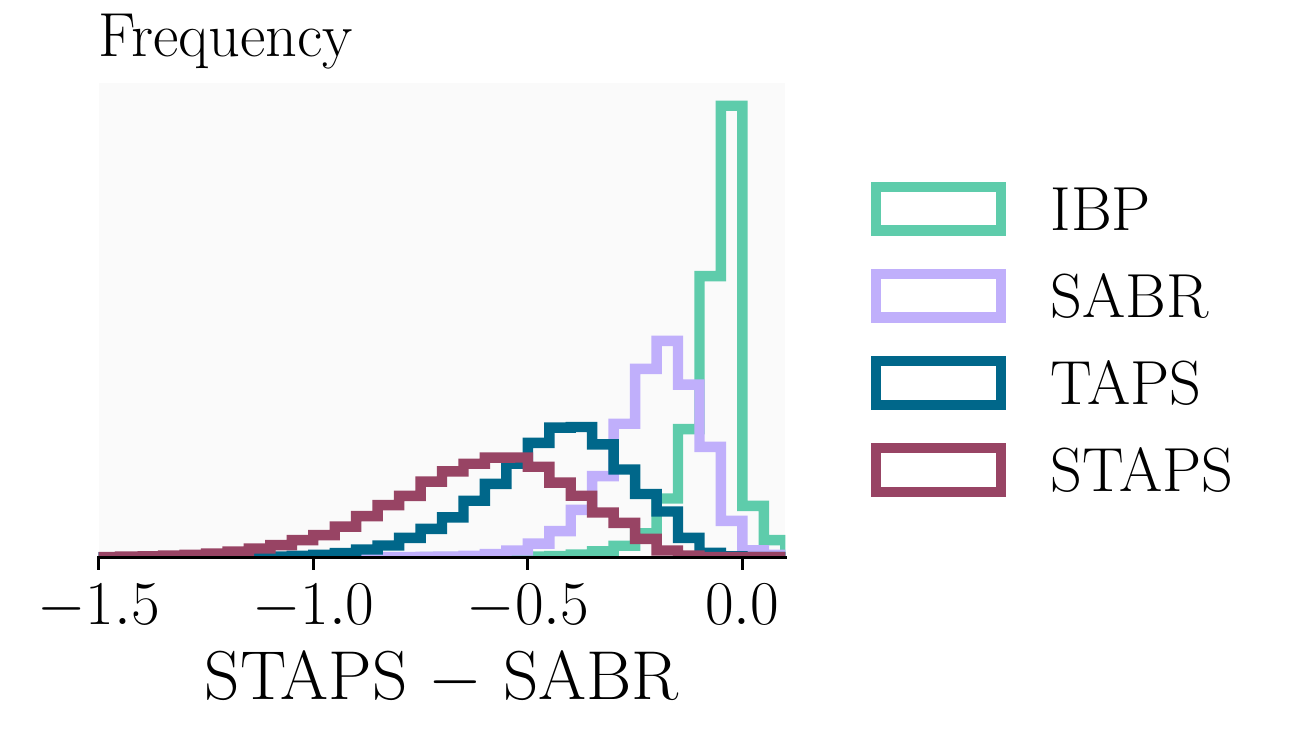}
        \vspace{-6mm}
        \caption{\toolp$\qquad\qquad\quad$}
    \end{subfigure}
    \vspace{-5mm}
    \caption{Bound difference between \ibp and \pgd propagation through the classifier depending on the training method.}
    \label{fig:taps_effect}
\end{wrapfigure}

\paragraph{Approximation Precision}
To evaluate whether \tool yields more precise approximations of the worst-case loss than other certified training methods, we compute approximations of the maximum margin loss with \ibp, \pgd ($50$ steps, $3$ restarts), \sabr ($\lambda=0.4$), and \tool on a small \tool-trained \cnnt for all \mnist test set samples. We report histograms over the difference to the exact worst-case loss computed with a \milp encoding \citep{TjengXT19} in \cref{fig:empirical_bound_tightness}. Positive values correspond to over-approximations while negative values correspond to under-approximation. We observe that the \tool approximation is by far the most precise, achieving the smallest mean and mean absolute error as well as variance. We confirm these observations for other training methods in \cref{fig:empirical_bound_tightness_ext} in \cref{app:extended_exp}.

% \begin{wrapfigure}[14]{r}{0.5 \textwidth}
%     \vspace{-4mm}
%     \centering
%     \begin{subfigure}[b]{0.41\linewidth}
%             \centering
%         \includegraphics[width=1.05\linewidth]{figures/TAPS_effect}
%         \vspace{-6mm}
%         \caption{\tool}
%     \end{subfigure}
%     \hfill
%     \begin{subfigure}[b]{0.574\linewidth}
%         \centering
%         \includegraphics[width=1.05\linewidth]{figures/STAPS_effect}
%         \vspace{-6mm}
%         \caption{\toolp$\qquad\qquad\quad$}
%     \end{subfigure}
%     \vspace{-5mm}
%     \caption{Under-approximation resulting from \ibp propagation through classifier component depending on training method.}
%     \label{fig:taps_effect}
% \end{wrapfigure}
To isolate the under-approximation effect of the \pgd propagation through the classifier, we visualize the distribution over pairwise bound differences between \tool and \ibp and \toolp and \sabr in \cref{fig:taps_effect} for different training methods. We observe that the distributions for \tool and \toolp are remarkably similar (up to scaling), highlighting the importance of reducing over-regularisation of the later layers, even when propagating only small regions (\sabr/\toolp). Further, we note that larger bound differences indicate reduced regularisation of the later network layers. We thus observe that \sabr still induces a much stronger regularisation of the later layers than \tool and especially \toolp, again highlighting the complementarity of \tool and \sabr, discussed in \cref{sec:STAPS}.

% \begin{table}[h]
%     \caption{Effect of \ibp regularization and the \tool gradient expanding coefficient $\alpha$ for MNIST $\epsilon=0.3$.}
%     \resizebox{0.98\linewidth}{!}{
%         \begin{tabular}{ccccccc}
%             \toprule
%             $\alpha$           & {Total time (s)} & {\#Tested} & {Avg time (s)} & Nat (\%) & {Adv. (\%)} & {Cert. (\%)} \\ 
%             \midrule
%             {0 (w/o \tool)}           & 23074      & 10000 & 2.3     & 97.6 & 93.37 & 93.15  \\
%             {0 (w/o \tool grad)}      & 26872      & 10000 & 2.7     & 97.37 & 93.32 & 93.06  \\
%             {1}                       & 45367      & 10000 & 4.5     & 97.86 & 93.80 & 93.36  \\
%             {5}                       & 68868      & 10000 & 6.9     & 98.16 & 94.18 & \textbf{93.55}      \\
%             {10}                      & 156973     & 10000 & 15.7    & 98.25 & 94.43 & 93.02  \\
%             {15}                      & >160000    & 3800  & 42.8    & 98.53 & \textbf{95.00}     & 91.55  \\
%             {20}                      & >160000    & 1200  & 73.7    & \textbf{98.75} & 94.33 & 82.67  \\
%             {$\infty$ (w/o \ibp grad)}& >160000    & 300   & 569.7   & 98.0 & 94.00 & 45.00  \\ 
%             {$\infty$ (w/o \ibp reg)} & >160000    & 200   & 817.1   & 98.5 & 94.50 & 17.50  \\
%             \bottomrule
%         \end{tabular}
%     }
% \end{table}

\begin{wraptable}[14]{r}{0.55 \textwidth}
	\centering
	\vspace{-4mm}
    \caption{Effect of \ibp regularization and the \tool gradient expanding coefficient $\alpha$ for MNIST $\epsilon=0.3$.}\label{tab:ibp_reg}
    \vspace{-1mm}
	\renewcommand{\arraystretch}{0.95}
	\begin{adjustbox}{width=0.95\linewidth,center}
		\begin{threeparttable}
		\begin{tabular}{@{}ccccccc@{}}
                \toprule
		$w_{\text{\tool}}$           & {Avg time (s)} & Nat (\%) & {Adv. (\%)} & {Cert. (\%)} \\ 
                \midrule
		{$\L_{\text{\ibp}}$}                     & 2.3     & 97.6 & 93.37 & 93.15  \\
		{0}                & 2.7     & 97.37 & 93.32 & 93.06  \\
		{1} $\,$                          & 4.5     & 97.86 & 93.80 & 93.36  \\
		{5} $\,$                          & 6.9     & 97.94 & 94.01 & \textbf{93.62}      \\
		{10}$\;\,$                          & 15.7    & 98.25 & 94.43 & 93.02  \\
		{15}$^\dagger$                      & 42.8    & 98.53 & \textbf{95.00}     & 91.55  \\
		{20}$^\dagger$                      & 73.7    & \textbf{98.75} & 94.33 & 82.67  \\
		{$\infty$}$^\dagger$& 569.7   & 98.0 & 94.00 & 45.00  \\ 
		{$\L_{\text{\tool}}$}$^\dagger$& 817.1   & 98.5 & 94.50 & 17.50  \\
                \bottomrule
            \end{tabular}
            \begin{tablenotes}
                \footnotesize
                   \item[$\dagger$] Only evaluated on part of the test set within a 2-day time limit.
            \end{tablenotes}
        \end{threeparttable}
    \end{adjustbox}
    % \vspace{-1mm}
\end{wraptable}

\paragraph{IBP Regularization}
To analyze the effectiveness of the multiplicative \ibp regularization discussed in \cref{sec:objective}, we train with \ibp in isolation ($\L_{\text{\ibp}}$), \ibp with \tool  weighted gradients ($w_\text{\tool}=0$), varying levels of gradient scaling for the \tool component ($w_\text{\tool} \in [1, 20]$), \tool with IBP weighting ($w_\text{\tool} = \infty$), and \tool loss in isolation, reporting results in \cref{tab:ibp_reg}.   
We observe that \ibp in isolation yields comparatively low standard but moderate certified accuracies with fast certification times. Increasing the weight $w_\text{\tool}$ of the \tool gradients reduces regularization, leading to longer certification times and higher standard accuracies. Initially, this translates to higher adversarial and certified accuracies, peaking at $w_\text{\tool}=15$ and $w_\text{\tool}=5$, respectively, before especially certified accuracy decreases as regularization becomes insufficient for certification. We confirm these trends for \TIN in \cref{tab:ibp_reg_tin} in \cref{app:extended_exp}.

\begin{wrapfigure}[15]{r}{0.70 \textwidth}
    \vspace{-4mm}
    \centering
    \begin{subfigure}[b]{0.33\linewidth}
            \centering
        \includegraphics[width=1.05\linewidth]{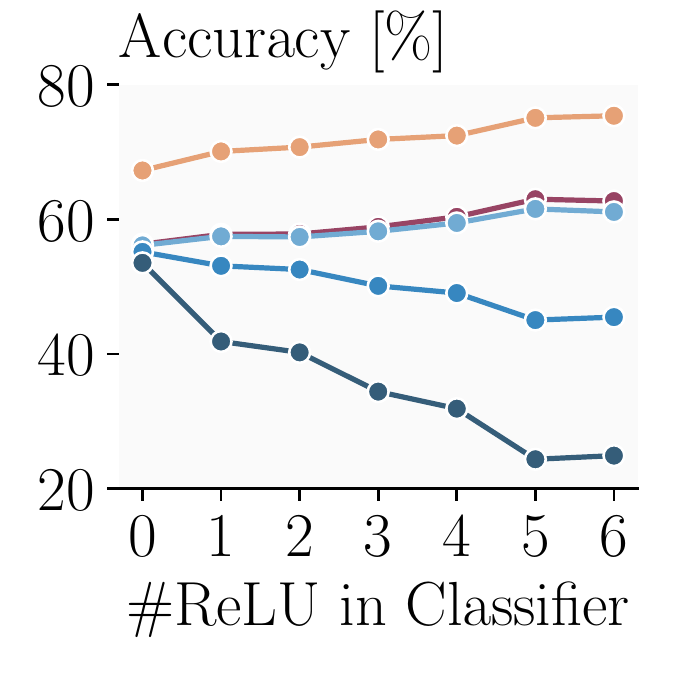}
        \vspace{-6mm}
        \caption{$\epsilon =2/255$}
    \end{subfigure}
    \hfill
    \begin{subfigure}[b]{0.65\linewidth}
        \centering
        \includegraphics[width=1.05\linewidth]{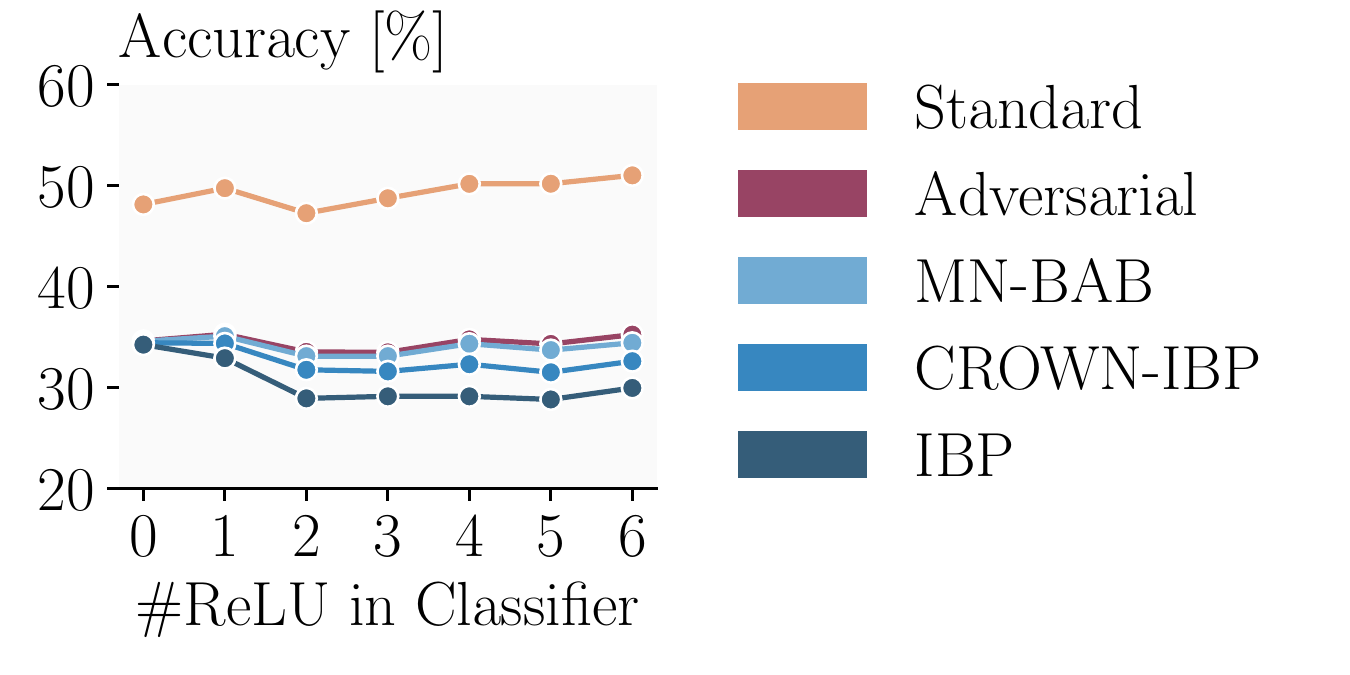}
        \vspace{-6mm}
        \caption{$\epsilon = 8/255 \qquad \qquad \qquad\;\;$}
        \label{fig:split_location_8}
    \end{subfigure}
    \vspace{-1mm}
    \caption{Effect of split location on the standard and robust accuracy of \tool trained networks, depending on the perturbation magnitude $\epsilon$ for different certification methods for \cifar. $0$ ReLUs in the classifier recovers \ibp training.}
    \label{fig:split_location}
    \vspace{1.2mm}
\end{wrapfigure}

\paragraph{Split Location} \label{sec:split}
\tool splits a given network into a feature extractor and classifier, which are then approximated using \ibp and \pgd, respectively. As \ibp propagation accumulates over-approximation errors while \pgd is an under-approximation, the location of this split has a strong impact on the regularization level induced by \tool.
To analyze this effect, we train multiple \cnnss such that we obtain classifier components with between $0$ and $6$ (all) ReLU layers and illustrate the resulting standard, adversarial, and certified (using different methods) accuracies in \cref{fig:split_location} for \cifar, and in \cref{app:extended_exp} for \mnist and \TIN in \cref{tab:classifier_size_mnist,tab:classifier_size_tin} respectively.

%\cref{tb:abalation_classifier_size_cifar,tb:abalation_classifier_size_mnist}.

For small perturbations ($\epsilon=2/255$), increasing classifier size and thus decreasing regularization yields increasing natural and adversarial accuracy. While the precise \mnbab verification can translate this to rising certified accuracies up to large classifier sizes, regularization quickly becomes insufficient for the less precise \ibp and \crownibp certification. 
For larger perturbations ($\epsilon=8/255$), the behavior is more complex. An initial increase of all accuracies with classifier size is followed by a sudden drop and slow recovery, with certified accuracies remaining below the level achieved for 1 ReLU layer.
We hypothesize that this effect is due to the \ibp regularization starting to dominate optimization combined with increased training difficulty (see \cref{app:extended_exp} for details).
For both perturbation magnitudes, gains in certified accuracy can only be realized with the precise \mnbab certification \citep{MuellerEFV22}, highlighting the importance of recent developments in neural network verification for certified training.

\begin{wrapfigure}[10]{r}{0.43 \textwidth}
    \centering
    \vspace{-0mm}
    \includegraphics[width=.49\linewidth]{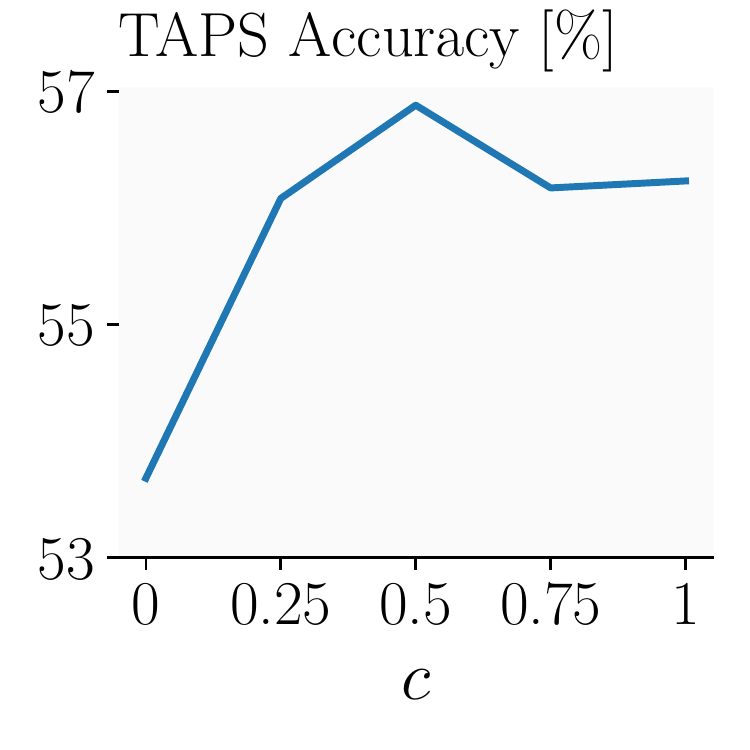}
    \hfil
    \includegraphics[width=.49\linewidth]{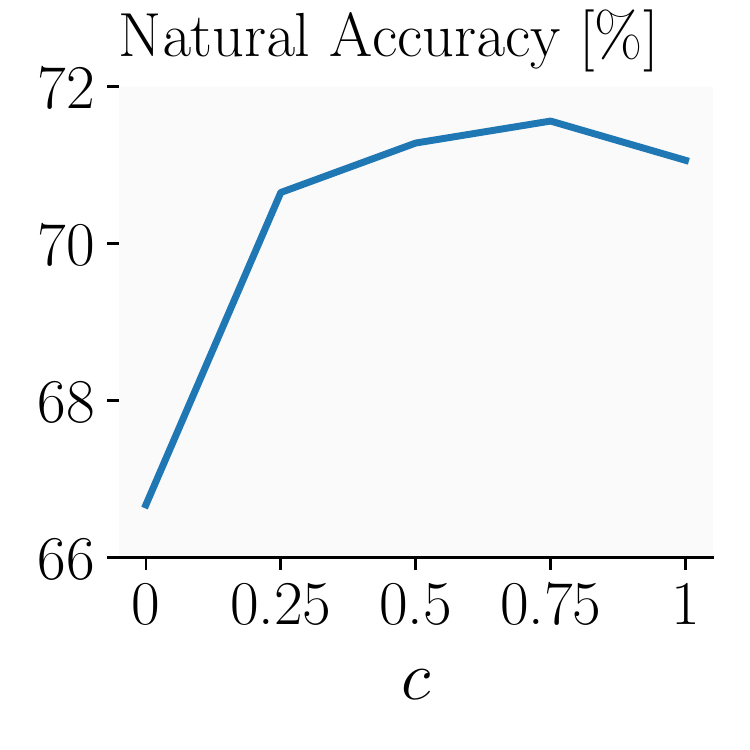}
    \vspace{-7mm}
    \caption{Effect of the gradient connector on \tool (left) and natural (right) accuracy.} \label{fig:grad_connector_ablation}
    % \vspace{-2mm}
\end{wrapfigure}
\paragraph{Gradient Connector} In \cref{fig:grad_connector_ablation}, we illustrate the effect of our gradient connector's parameterization c (\cref{sec:connector}). We report \tool accuracy (the portion of samples where all latent adversarial examples are classified correctly) as a proxy for the goodness of fit. Recall that $c=0$ corresponds to the binary connector and $c=1$ to the linear connector. We observe that the binary connector achieves poor \tool and natural accuracy, indicating a less well-behaved optimization problem. \tool accuracy peaks at $c=0.5$, indicating high goodness-of-fit and thus a well-behaved optimization problem. This agrees well with our theoretical considerations aiming to avoid sparsity ($c<0.5$) and contradicting gradients ($c>0.5$).

\begin{wraptable}[11]{r}{0.49\textwidth}
    \centering
    \vspace{-4.9mm}
    \caption{Comparison of single- and multi-estimator \pgd, depending on the split position for \mnist at $\eps=0.3$.} \label{tb:single_multi}
    \vspace{-1mm}
	\begin{adjustbox}{width=0.95\linewidth,center}
		\begin{threeparttable}
            \begin{tabular}{ccccc}
            \toprule
            \multirow{2.4}{*}{\makecell{\# $\relu$ in\\Classifier}} & \multicolumn{2}{c}{Single}      & \multicolumn{2}{c}{Multi}    \\
            \cmidrule(rl){2-3}				\cmidrule(rl){4-5}
                                     & Certified      & Natural        & Certified      & Natural        \\ \midrule
            1                             & -$^\dagger$     & 31.47$^\dagger$ & \textbf{93.62} & 97.94 \\
            3                            & 92.91 & 98.56 & 93.03 & 98.63 \\
            6                            & 92.41 & \textbf{98.88} & 92.70 & \textbf{98.88} \\ \bottomrule
            \end{tabular}
            \begin{tablenotes}
                \footnotesize
                \item[$\dagger$] Training encounters mode collapse. Last epoch performance reported.
            \end{tablenotes}
        \end{threeparttable}
    \end{adjustbox}
    % \vspace{-2mm}
\end{wraptable}

\paragraph{Single-Estimator vs Multi-Estimator PGD}
To evaluate the importance of our multi-estimator \pgd variant, we compare it to single-estimator \pgd across a range of split positions, reporting results in  \cref{tb:single_multi}. We observe that across all split positions, multi-estimator \pgd achieves better certified and better or equal natural accuracy. Further, training collapses reproducibly for single-estimator \pgd for small classifiers, indicating that multi-estimator \pgd additionally improves training stability.
% \todo{Multiestimator PGD for adversarial training?}

\begin{wraptable}[9]{r}{0.49\textwidth}
    \centering
    \vspace{-4.7mm}
    \caption{Effect of different \pgd attack strengths for \mnist at $\eps=0.3$.} \label{tb:attack_strength}
    \vspace{-1mm}
	\begin{adjustbox}{width=0.95\linewidth,center}
		\begin{threeparttable}
            \begin{tabular}{ccccc}
            \toprule
            \multirow{2.4}{*}{\makecell{\# Attack Steps}} & \multicolumn{2}{c}{1 Restart}      & \multicolumn{2}{c}{3 Restarts}    \\
            \cmidrule(rl){2-3}				\cmidrule(rl){4-5}
                                        & Certified      & Natural        & Certified      & Natural        \\ \midrule
            1                           & 93.36 & \textbf{98.22} & 93.47 & \textbf{98.22} \\
            5                           & 93.15 & 97.90 & \textbf{93.55} & 97.90 \\
            20                          & \textbf{93.62} & 97.94 & 93.52 & 97.99 \\
            100                         & 93.46 & 97.94 & \textbf{93.55} & 97.99 \\ 
            \bottomrule
            \end{tabular}
        \end{threeparttable}
    \end{adjustbox}
    % \vspace{-2mm}
\end{wraptable}

\paragraph{PGD Attack Strength}
To investigate the effect of the adversarial attack's strength, we use $1$ or $3$ restarts and vary the number of attack steps used in \tool from $1$ to $100$ for \mnist at $\eps=0.3$, reporting results in \cref{tb:attack_strength}. 
Interestingly, even a single attack step and restart are sufficient to achieve good performance and outperform \ibp. 
As we increase the strength of the attack, we can increase certified accuracy slightly while marginally reducing natural accuracy, agreeing well with our expectation of regularization strength increasing with attack strength.

% \paragraph{Combination with \sabr}

\section{Related Work} \label{sec:related_work}

\paragraph{Verification Methods} In this work, we only consider deterministic verification methods, which analyze a given network as is. While \emph{complete} (or \emph{exact}) methods \citep{TjengXT19,WangZXLJHK21,ZhangWXLLJ22,FerrariMJV22} can decide any robustness property given enough time, \emph{incomplete} methods \citep{SinghGMPV18,RaghunathanSL18,ZhangWCHD18,DathathriDKRUBS20,MullerMSPV22} sacrifice some precision for better scalability. However, recent complete methods can be used with a timeout to obtain effective incomplete methods.
% In this work, we only use deterministic methods.

\paragraph{Certified Training} 
% Certified training methods are used to obtain certifiably robust networks.
Most certified training methods compute and minimize sound over-approximations of the worst-case loss using different approximation methods: \diffai \citep{MirmanGV18} and \ibp \citep{GowalIBP2018} use \boxd approximations, \citet{WongSMK18} use \deepz relaxations \citep{SinghGMPV18}, \citet{WongK18} back-substitute linear bounds using fixed relaxations, \citet{ZhangCXGSLBH20} use dynamic relaxations \citep{ZhangWCHD18,SinghGPV19} and compute intermediate bounds using \boxd relaxations.
\citet{ShiWZYH21} significantly shorten training schedules by combining \ibp training with a special initialization.
Some more recent methods instead compute and optimize more precise, but not necessarily sound, worst-case loss approximations:
\sabr \citep{MuellerEFV22} reduce the regularization of \ibp training by propagating only small but carefully selected subregions.
\ibpr \citep{PalmaIBPR22} combines adversarial training at large perturbation radii with an \ibp-based regularization. 
\colt \citep{BalunovicV20} is conceptually most similar to \tool and thus compared to in more detail below. While prior work combined a robust and a precise network \citep{MuellerBV21,HorvathMFV22}, to trade-off certified and standard accuracy, these unsound certified training methods can often increase both.

\colt \citep{BalunovicV20}, similar to \tool, splits the network into a feature extractor and classifier, computing bounds on the feature extractor's output (using the \zono \citep{SinghGPV19} instead of \boxd domain) before conducting adversarial training over the resulting region. Crucially, however, \colt lacks a gradient connector and, thus, does not enable gradient flow between the latent adversarial examples and the bounds on the feature extractor's output. Therefore, gradients can only be computed for the weights of the classifier but not the feature extractor, preventing the two components from being trained jointly. Instead, a stagewise training process is used, where the split between feature extractor and classifier gradually moves through the network starting with the whole network being treated as the classifier. This has several repercussions: not only is the training very slow and limited to relatively small networks (a four-layer network takes almost 2 days to train) but more importantly, the feature extractor (and thus the whole network) is never trained specifically for precise bound propagation. Instead, only the classifier is trained to become robust to the incurred imprecisions. As this makes bound propagation methods ineffective for certification, \citet{BalunovicV20} employ precise but very expensive mixed integer linear programming (MILP \citep{TjengXT19}), further limiting the scalability of \colt.

% \todo{continue}

% we compare in more detail: \colt uses the tighter \zono \citep{SinghGPV19} instead of \boxd bounds combined with adversarial training to reduce approximation errors:
% \citet{BalunovicV20} split a network into two parts $\vf = \vf_2 \circ \vf_1$. During every stage of training, they now fix the weights $\w_1$ of $\vf_1$ and only train $\vf_2$ as follows: 
% First, they compute a \zono over-approximation of the reachable set of $\vf_1$. Then, they conduct adversarial training of the remaining network $\vf_2$ within the thus obtained bounds. 
% They repeat this process in stages, at first assigning the whole network to $\vf_2$, before slowly moving this interface between \zono and \pgd propagation through the network.
% Crucially, they do not facilitate the computation of any gradients with respect to $\vf_1$. Thus the weights of $\w_1$ are effectively frozen and $\vf_2$ is trained to be robust to approximation errors of $\vf_1$. However, $\vf_1$ is \emph{not} trained to minimize these approximation errors. This makes training very slow and limited to small networks (even for certified training methods). \todo{expand this}

In our experimental evaluation (\cref{sec:main_result}), we compare \tool in detail to the above methods.

\paragraph{Robustness by Construction}
\citet{LiCWC19}, \citet{LecuyerAG0J19}, and \citet{CohenRK19} construct probabilistic classifiers by introducing randomness into the inference process of a base classifier. This allows them to derive robustness guarantees with high probability at the cost of significant (100x) runtime penalties. \citet{SalmanLRZZBY19} train the base classifier using adversarial training and \citet{HorvathMFV22A} ensemble multiple base models to improve accuracies at a further runtime penalty.
\citet{ZhangCLHW21,ZhangJH022} introduce $\ell_\infty$-distance neurons, generalized to \sortnet by \citet{ZhangJHW22} which inherently exhibits $\ell_\infty$-Lipschitzness properties, yielding good robustness for large perturbation radii, but poor performance for smaller ones.
%\citep{SalmanLRZZBY19} 

\section{Conclusion}

We propose \tool, a novel certified training method that reduces over-regularization by constructing and optimizing a precise worst-case loss approximation based on a combination of \ibp and \pgd training. 
Crucially, \tool enables joint training over the \ibp and \pgd approximated components by introducing the gradient connector to define a gradient flow through their interface. 
Empirically, we confirm that \tool yields much more precise approximations of the worst-case loss than existing methods and demonstrate that this translates to state-of-the-art performance in certified training in many settings.

\section*{Acknowledgements}
We would like to thank our anonymous reviewers for their constructive comments and insightful questions.

This work has been done as part of the EU grant ELSA (European Lighthouse on Secure and Safe AI, grant agreement no. 101070617) and the SERI grant SAFEAI (Certified Safe, Fair and Robust Artificial Intelligence, contract no. MB22.00088). Views and opinions expressed are however those of the authors only and do not necessarily reflect those of the European Union or European Commission. Neither the European Union nor the European Commission can be held responsible for them. 

The work has received funding from the Swiss State Secretariat for Education, Research and Innovation (SERI).

%%%%%%%%%%%%%%
% BODY - END %
%%%%%%%%%%%%%%
% this message marks the end of the body (used to check if we are over the page
% limit)
\message{^^JLASTBODYPAGE \thepage^^J}

%%%%%%%%%%%%%%%%
% BIBLIOGRAPHY %
%%%%%%%%%%%%%%%%
% may be replaced by new bibliography
\clearpage
\bibliography{references}
\bibliographystyle{IEEEtranN}

%%%%%%%%%%%%%%%%%%%%
% BIBLIOGRAPHY END %
%%%%%%%%%%%%%%%%%%%%
% this message marks the end of the bibliography (used to split the paper from
% the supplement)
\message{^^JLASTREFERENCESPAGE \thepage^^J}

%%%%%%%%%%%%
% APPENDIX %
%%%%%%%%%%%%

% may decide to not include appendix
\ifbool{includeappendix}{%
	\clearpage
	\appendix
	\section{Averaging Multipliers Makes Gradients Efficient} \label{app:proof_ibp_reg}

\begin{theorem} \label{thm:ibp_reg}
    Let $x_i$ be i.i.d. drawn from the dataset and define $f_i = f_\theta(x_i)$ and $g_i = g_\theta(x_i)$, where $f_\theta$ and $g_\theta$ are two functions. Further, define $L_1 = (\sum_{i=1}^n \frac{1}{n} f_i)\cdot (\sum_{i=1}^n \frac{1}{n} g_i)$ and $L_2 = \sum_{i=1}^n \frac{1}{n} f_i g_i$. Then, assuming the function value and the gradient are independent, $\E_x \left(\frac{\partial L_1}{\partial \theta}\right) = \E_x \left(\frac{\partial L_2}{\partial \theta}\right)$ and $\Var_x \left(\frac{\partial L_1}{\partial \theta}\right) \le \Var_x \left(\frac{\partial L_2}{\partial \theta}\right)$.
\end{theorem}

\begin{proof}
    A famous result in stochastic optimization is that stochastic gradients are unbiased. For completeness, we give a short proof of this property: Let $L = \E_{x} f(x) = \int_{-\infty}^{+\infty} f(x) dP(x)$, thus $\nabla_x L = \nabla_x (\int_{-\infty}^{+\infty} f(x) dP(x)) = \int_{-\infty}^{+\infty} \nabla_x f(x) dP(x) = \E_x (\nabla_x f(x))$. Therefore, $\nabla f(x_i)$ is an unbiased estimator of the true gradient.

    Applying that the stochastic gradients are unbiased, we can write $\nabla_\theta f_i = \nabla_\theta f + \eta_i$, where $\nabla_\theta f$ is the expectation of the gradient and $\eta_i$ is the deviation such that $\E \eta_i = 0$ and $\Var (\eta_i) = \sigma_1^2$. Since $x_i$ is drawn independently, $f_i$ are independent and thus $\eta_i$ are independent. Similarly, we can write $\nabla_\theta g_i = \nabla_\theta g + \delta_i$, where $\E \delta_i = 0$ and $\Var(\delta_i) = \sigma_2^2$. $\eta_i$ and $\delta_i$ may be dependent.

    Define $\bar{f} = \sum_i \frac{1}{n} f_i$ and $\bar{g} = \sum_i \frac{1}{n} g_i$. Explicit computation gives us that
    $
        \nabla L_1 = \bar{g} \cdot \left(\sum_i \frac{1}{n} \nabla f_i\right) + \bar{f} \cdot \left(\sum_i \frac{1}{n} \nabla g_i\right),
    $
    and
    $
        \nabla L_2 = \sum_i \frac{1}{n} \left(f_i \nabla g_i + g_i \nabla f_i\right)
    $. Therefore,
    $$
        \E_x \left(\nabla_\theta L_1 \mid f_i, g_i\right)
        = \bar{g} \nabla_\theta f + \bar{f} \nabla_\theta g
        = \E_x \left(\nabla_\theta L_2 \mid f_i, g_i\right).
    $$
    By the law of total probability,
    \begin{align*}
        \E_x \left(\nabla_\theta L_1 \right)
         & = \E_{f_i, g_i} \left(\E_x \left(\nabla_\theta L_1 \mid f_i, g_i \right) \right) \\
         & = \E_{f_i, g_i} \left(\E_x \left(\nabla_\theta L_2 \mid f_i, g_i \right) \right) \\
         & = \E_x \left(\nabla_\theta L_2 \right).
    \end{align*}
    Therefore, we have got the first result: the gradients of $L_1$ and $L_2$ have the same expectation.

    To prove the variance inequality, we will use variance decomposition formula\footnote{https://en.wikipedia.org/wiki/Law\_of\_total\_variance}:
    \begin{equation*}
        \begin{split}
            \Var_x (\nabla_\theta L_k) =  \E_{f_i, g_i} (\Var_x(\nabla_\theta L_k \mid f_i, g_i)) + \\ \Var_{f_i, g_i} (\E_x(\nabla_\theta L_k \mid f_i, g_i)),
        \end{split}
    \end{equation*}
    $k=1,2$. We have proved that $\E_x(\nabla_\theta L_1 \mid f_i, g_i) = \E_x(\nabla_\theta L_2 \mid f_i, g_i)$, thus the second term is equal. Next, we prove that $\Var_x(\nabla_\theta L_1 \mid f_i, g_i) \le \Var_x(\nabla_\theta L_2 \mid f_i, g_i)$, which implies $\Var_x (\nabla_\theta L_1) \le \Var_x (\nabla_\theta L_2)$.

    By explicit computation, we have
    \begin{align}
         & \Var(\nabla L_1 \mid f_i, g_i) \nonumber                                                                                       \\
         & = (\bar{g})^2 \Var\left(\sum_i \frac{1}{n} \eta_i\right) + (\bar{f})^2 \Var \left(\sum_i \frac{1}{n} \delta_i\right) \nonumber \\
         & = \frac{1}{n} \sigma_1^2 (\bar{g})^2 + \frac{1}{n} \sigma_2^2 (\bar{f})^2, \label{eq:ibp2}
    \end{align}
    and
    \begin{align}
         & \Var(\nabla L_2 \mid f_i, g_i) \nonumber                                                                                                       \\
         & = \Var\left(\sum_i \frac{1}{n} f_i\delta_i\right) + \Var\left(\sum_i \frac{1}{n} g_i\eta_i\right) \nonumber                                    \\
         & = \frac{1}{n} \sigma_1^2 \left(\sum_i \frac{1}{n} g_i^2\right) + \frac{1}{n} \sigma_2^2 \left(\sum_i \frac{1}{n} f_i^2\right). \label{eq:ibp3}
    \end{align}
    Applying Jensen's formula on the convex function $x^2$, we have $\left(\sum_i \frac{1}{n} a_i\right)^2 \le \sum_i \frac{1}{n} a_i^2$ for any $a_i$, thus $(\bar{f})^2 \le \sum_i \frac{1}{n} f_i^2$ and $(\bar{g})^2 \le \sum_i \frac{1}{n} g_i^2$. Combining \cref{eq:ibp2} and \cref{eq:ibp3} with these two inequalities gives the desired result.
\end{proof}

\section{Experiment Details} \label{sec:exp_setting}

\begin{algorithm}[t]
    \caption{Train Loss Computation}
    \label{alg:train}
    \begin{algorithmic}
        \State {\bfseries Input:} data $X_{B} = \{(\vx_b, y_b)\}_b$, current $\epsilon$, target $\epsilon^t$, network $\vf$
        \State {\bfseries Output:} A differentiable loss $L$
        \State $\bc{L}_{\ibp} = \sum_{b \in \mathcal{B}} \bc{L}_\ibp(\vx_b, y_b, \epsilon) / |\mathcal{B}|$.
        \If{$\epsilon < \epsilon^t$}
        \State // $\epsilon$ annealing regularisation from \citet{ShiWZYH21}
        \State $\bc{L}_{\text{fast}} = \lambda \cdot (\bc{L}_{\text{tightness}} + \bc{L}_{\text{relu}})$
        \State {\bfseries return} $\bc{L}_{\text{IBP}} + \epsilon / \epsilon^t \cdot \bc{L}_{\text{fast}}$
        \EndIf
        \State $\bc{L}_{\tool} = \sum_{b \in \mathcal{B}} L_\tool(\vx_b, y_b, \epsilon) / |\mathcal{B}|$.
        \State {\bfseries return} $\bc{L}_{\text{IBP}} \cdot \bc{L}_{\text{TAPS}}$
    \end{algorithmic}
\end{algorithm}

\subsection{TAPS Training Procedure}
To obtain state-of-the-art performance with \ibp, various training techniques have been developed. We use two of them: $\epsilon$-annealing \citep{GowalIBP2018} and initialization and regularization for stable box sizes \citep{ShiWZYH21}. $\epsilon$-annealing slowly increases the perturbation magnitude $\epsilon$ during training to avoid exploding approximation sizes and thus gradients. The initialization of \citet{ShiWZYH21} scales network weights to achieve constant box sizes over network depth.
During the $\epsilon$-annealing phase, we combine the \ibp loss with the ReLU stability regularization $\bc{L}_{\text{fast}}$ \citep{ShiWZYH21}, before switching to the \tool loss as described in \cref{sec:objective}. We formalize this in \cref{alg:train}. We follow \citet{ShiWZYH21} in doing early stopping based on validation set performance. However, we use \tool accuracy (see \cref{app:extended_exp}) instead of \ibp accuracy as a performance metric.

\subsection{Datasets and Augmentation}

We use the \mnist \citep{lecun2010mnist}, \cifar \citep{krizhevsky2009learning}, and \TIN \citep{Le2015TinyIV} datasets, all of which are freely available with no license specified.

The data preprocessing mostly follows \citet{MuellerEFV22}. For \mnist, we do not apply any preprocessing. For \cifar and \TIN, we normalize with the dataset mean and standard deviation (after calculating perturbation size) and augment with random horizontal flips. For \cifar, we apply random cropping to $32 \times 32$ after applying a $2$ pixel padding at every margin. For \TIN, we apply random cropping to $56\times 56$ during training and center cropping during testing.

\subsection{Model Architectures}

Unless specified otherwise, we follow \citet{ShiWZYH21,MuellerEFV22} and use a \cnns with Batch Norm for our main experiments. \cnns is a convolutional network with $7$ convolutional and linear layers. All but the last linear layer are followed by a Batch Norm and ReLU layer. %It has 1 normalization layer, 5 convolutional blocks with 1 convolutional layer, 1 BN layer and 1 ReLU layer for each block, followed by 1 linear, 1 BN, 1 ReLU and 1 additional linear layer. Therefore, it has 21 layers in total.

\subsection{Training Hyperparameter Details}

We follow the hyperparameter choices of \citet{ShiWZYH21} for $\epsilon$-annealing, learning rate schedules, batch sizes, and gradient clipping (see \cref{tb:epoch}). We set the initial learning rate to 0.0005 and decrease it by a factor of $0.2$ at Decay-1 and -2. We set the gradient clipping threshold to 10.
\begin{table}[t]
    \centering
    \caption{The training epoch and learning rate settings.} \label{tb:epoch}
    \resizebox{0.85\linewidth}{!}{
    \begin{tabular}{ccS[table-format=3]cS[table-format=3]S[table-format=3]}
        \toprule
        Dataset & \text{Batch size} & \text{Total epochs} & \text{Annealing epochs} & \text{Decay-1} & \text{Decay-2} \\
        \midrule
        \mnist  & 256               & 70                  & 20                      & 50             & 60             \\
        \cifar  & 128               & 160                 & 80                      & 120            & 140            \\
        \TIN    & 128               & 80                  & 20                      & 60             & 70             \\
        \bottomrule
    \end{tabular}
    }
\end{table}

We use additional $L_1$ regularization in some settings where we observe signs of overfitting. We report the $L_1$ regularization and split position chosen for different settings in \cref{tb:hyperparameter_TAPS} and \cref{tb:hyperparameter_STAPS}.

We train using single NVIDIA GeForce RTX 3090 for \mnist and \cifar and single NVIDIA TITAN RTX for \TIN. Training and certification times are reported in \cref{tb:time_TAPS} and \cref{tb:time_STAPS}.

\begin{table}[t]
    \centering
    \caption{Hyperparameters for \tool.} \label{tb:hyperparameter_TAPS}
    \resizebox{0.5 \linewidth}{!}{
        \begin{tabular}{ccccc}
            \toprule
            \multicolumn{1}{c}{Dataset} & $\epsilon$ & \# ReLUs in Classifier & $L_1$ & $w$ \\ \midrule
            \multirow{2}{*}{\mnist}     & 0.1        & 3                      & 1e-6  & 5   \\
                                        & 0.3        & 1                      & 0     & 5   \\
            \multirow{2}{*}{\cifar}     & 2/255      & 5                      & 2e-6  & 5   \\
                                        & 8/255      & 1                      & 2e-6  & 5   \\
            \TIN                        & 1/255      & 1                      & 0     & 5   \\ \bottomrule
        \end{tabular}
    }
\end{table}

\begin{table}[t]
    \centering
    \caption{Training and certification times for \tool-trained networks.} \label{tb:time_TAPS}
    \resizebox{0.5 \linewidth}{!}{
        \begin{tabular}{ccccS[table-number-alignment = left,table-format=5]S[table-number-alignment = left,table-format=5]}
            \toprule
            \multicolumn{1}{c}{Dataset} & $\epsilon$ & \text{Train Time (s)} & \text{Certify Time (s)} \\ \midrule
            \multirow{2}{*}{\mnist}     & 0.1        & 42$\,$622             & 17$\,$117               \\
                                        & 0.3        & 12$\,$417             & 41$\,$624               \\
            \multirow{2}{*}{\cifar}     & 2/255      & 141$\,$281            & 166$\,$474              \\
                                        & 8/255      & 27$\,$017             & 26$\,$968               \\
            \TIN                        & 1/255      & 306$\,$036            & 23$\,$497               \\ \bottomrule
        \end{tabular}
    }
\end{table}

\begin{table}[t]
    \centering
    \caption{Hyperparameter for \toolp.} \label{tb:hyperparameter_STAPS}
    \resizebox{0.5\linewidth}{!}{
        \begin{tabular}{cccccc}
            \toprule
            \multicolumn{1}{c}{Dataset} & $\epsilon$ & \# ReLUs in Classifier & $L_1$ & $w$ & $\tau/\epsilon$ \\
            \midrule
            \multirow{2}{*}{\mnist}     & 0.1        & 1                      & 2e-5  & 5 & 0.4 \\
                                        & 0.3        & 1                      & 2e-6  & 5 & 0.6 \\
            \multirow{2}{*}{\cifar}     & 2/255      & 1                      & 2e-6  & 2 & 0.1 \\
                                        & 8/255      & 1                      & 2e-6  & 5 & 0.7 \\
            \TIN                        & 1/255      & 2                      & 1e-6  & 5 & 0.6 \\
            \bottomrule
        \end{tabular}
    }
\end{table}

\begin{table}[t]
    \centering
    \caption{Training and certification times for \toolp-trained networks.} \label{tb:time_STAPS}
    \resizebox{0.5\linewidth}{!}{
        \begin{tabular}{ccS[table-number-alignment = left,table-format=5]S[table-number-alignment = left,table-format=5]}
            \toprule
            \multicolumn{1}{c}{Dataset} & $\epsilon$ & \text{Train Time (s)} & \text{Certify Time (s)} \\
            \midrule
            \multirow{2}{*}{\mnist}     & 0.1        & 19$\,$865             & 12$\,$943               \\
                                        & 0.3        & 23$\,$613             & 125$\,$768              \\
            \multirow{2}{*}{\cifar}     & 2/255      & 47$\,$631             & 398$\,$245              \\
                                        & 8/255      & 48$\,$706             & 77$\,$793               \\
            \TIN                        & 1/255      & 861$\,$639            & 35$\,$183               \\
            \bottomrule
        \end{tabular}
    }
    \vspace{-2mm}
\end{table}

\subsection{Certification Details}

We combine \ibp \citep{GowalIBP2018}, \crownibp\citep{ZhangCXGSLBH20}, and \mnbab\citep{FerrariMJV22} for certification, running the most precise but also computationally costly \mnbab only on samples not certified by the other methods. We use the same configuration for \mnbab as \citet{MuellerEFV22}. The certification is run on a single NVIDIA TITAN RTX.

\mnbab \cite{FerrariMJV22} is a state-of-the-art \citep{BrixMBJL22,MuellerBBLJ22} neural network verifier, combining the branch-and-bound paradigm \citep{BunelLTTKK20} with precise multi-neuron constraints \citep{MullerMSPV22,SinghGPV19B}. 

We use a mixture of strong adversarial attacks to evaluate adversarial accuracy. First, we run PGD attacks with 5 restarts and 200 iterations each. Then, we run MN-BaB to search for adversarial examples with a timeout of $1000$ seconds.

% \begin{figure}[ht]
%     \centering
%     \begin{subfigure}[b]{0.4\linewidth}
%         \centering
%         \includegraphics[width=1.05\linewidth]{figures/cifar_tightness_methods.pdf}
%         \vspace{-6mm}
%         \caption{Zoom-out}
%     \end{subfigure}
%     \hfill
%     \begin{subfigure}[b]{0.57\linewidth}
%         \centering
%         \includegraphics[width=1.05\linewidth]{figures/cifar_tightness_methods_enlarged.pdf}
%         \vspace{-6mm}
%         \caption{Zoom-in$\qquad\quad\;$}
%     \end{subfigure}
%     \vspace{-3mm}
%     \caption{Tightness coefficient for networks with different width.}
%     \label{fig:PIC_width}
%     \vspace{-2mm}
% \end{figure}

\section{Extended Evaluation} 
\label{app:extended_exp}

% \subsection{Selecting TAPS-Trained Models}

\begin{figure*}[t]
    % \vspace{9.75mm}
    \centering
    \begin{subfigure}[b]{0.25\linewidth}
        \centering
        \includegraphics[width=1.03\linewidth]{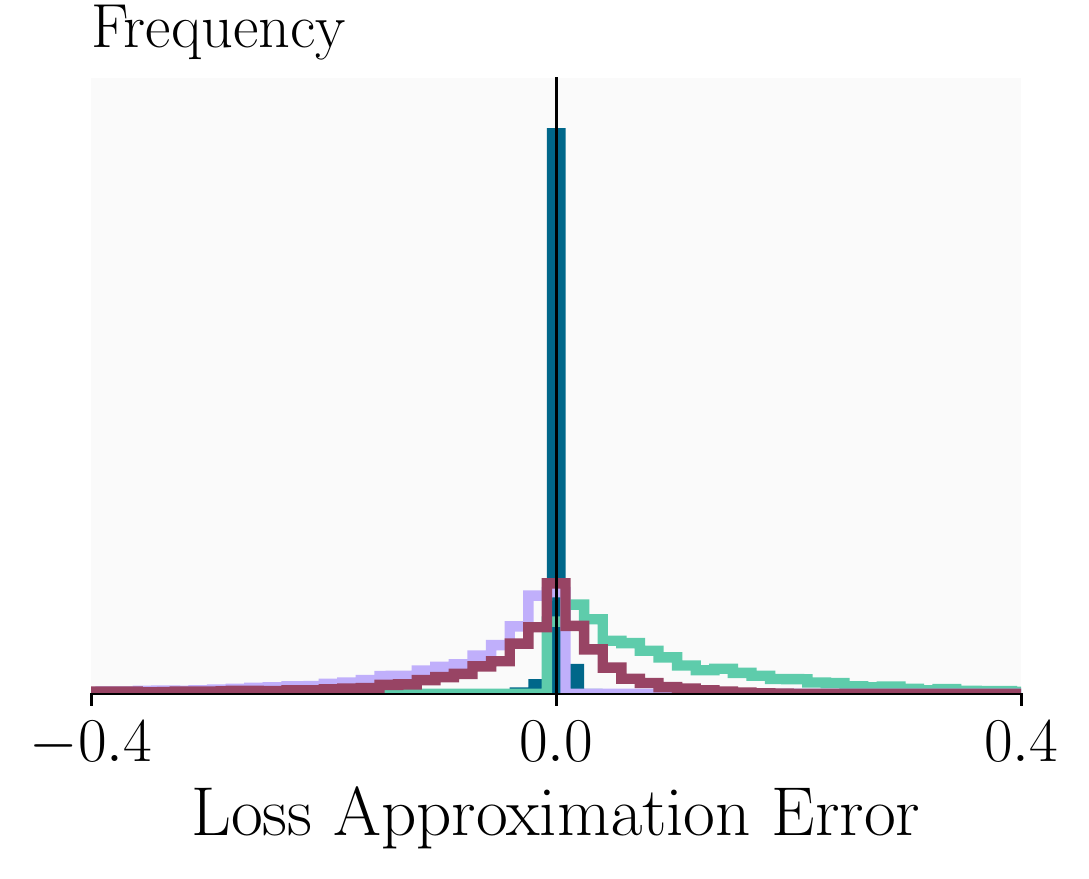}
        \vspace{-6mm}
        \caption{IBP-trained}
    \end{subfigure}
    \hfill
    \begin{subfigure}[b]{0.25\linewidth}
        \centering
        \includegraphics[width=1.03\linewidth]{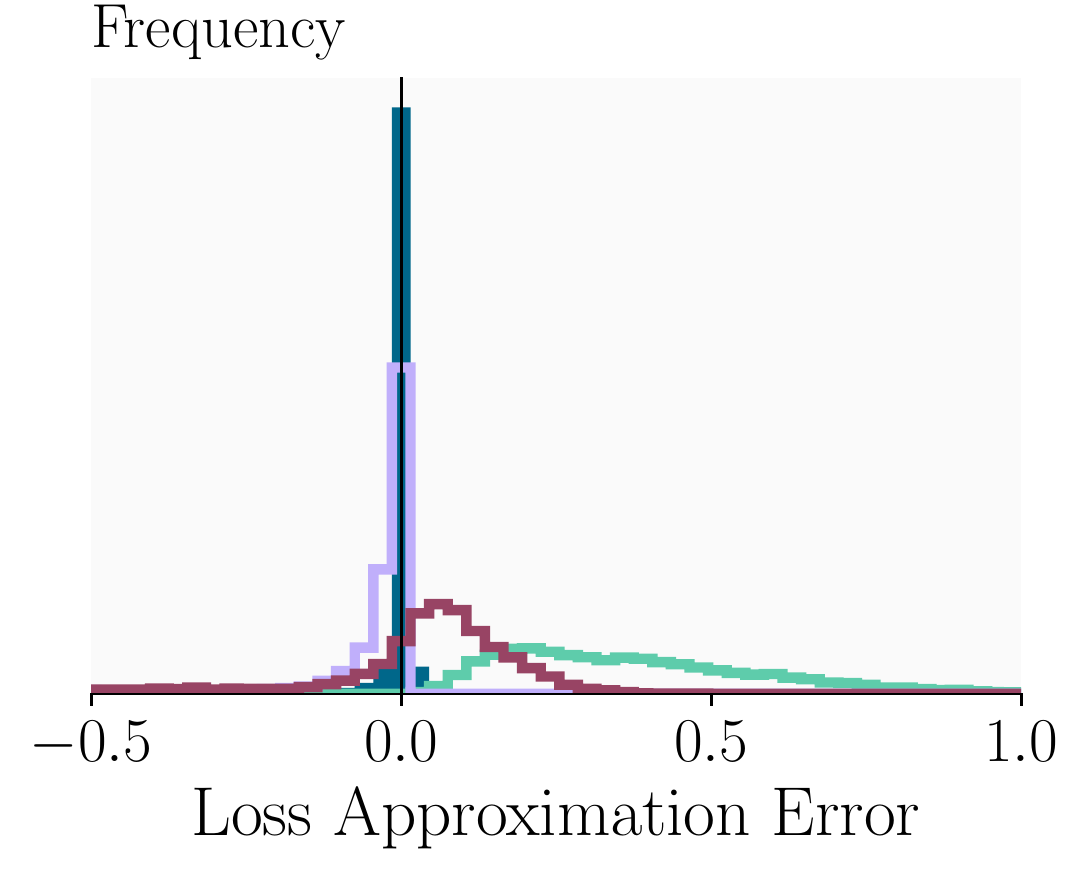}
        \vspace{-6mm}
        \caption{\sabr-trained}
    \end{subfigure}
    \hfill
    \begin{subfigure}[b]{0.42\linewidth}
        \centering
        \includegraphics[width=1.03\linewidth]{figures/exact_tightness_TAPS.pdf}
        \vspace{-6mm}
        \caption{\tool-trained$\qquad\quad\;\;\;\;\;\;$}
    \end{subfigure}
    \vspace{-2mm}
    \caption{Distribution of the worst-case loss approximation errors over test set samples, depending on the training and bounding method. Positive values correspond to over-approximations and negative values to under-approximations. We use an exact MILP encoding \citep{TjengXT19} as reference.}
    \label{fig:empirical_bound_tightness_ext}
    \vspace{-3mm}
\end{figure*}

\paragraph{\tool Accuracy as GoF}
In practice, we want to avoid certifying every model with expensive certification methods, especially during hyperparameter tuning and for early stopping. Therefore, we need a criterion to select models. In this section, we aim to show that \tool accuracy (accuracy of the latent adversarial examples) is a good proxy for goodness of fit (GoF).

We compare the \tool accuracy to adversarial and certified accuracy with all models we get on \mnist and \cifar. The result is shown in \cref{tb:TB_adv_cert}. We can see that the correlations between \tool accuracy and both the adversarial and the certified accuracy are close to 1. In addition, the differences are small and centered at zero, with a small standard deviation. Therefore, we conclude that \tool accuracy is a good estimate of the true robustness, thus a good measurement of GoF. In all the experiments, we perform model selection based on the \tool accuracy.

\begin{table}[t]
    \centering
    \caption{Comparison of \tool accuracy with certified and adversarial accuracy.} 
    \label{tb:TB_adv_cert}
    \resizebox{0.8\linewidth}{!}{
        \begin{tabular}{ccccc} \toprule
            Dataset  & cor(\tool, cert.) & cor(\tool, adv.) & \tool $-$ cert. & \tool $-$ adv. \\ \midrule
            \mnist   & 0.9139                    & 0.9633                   & 0.0122 $\pm$ 0.0141     & 0.0033 $\pm$ 0.0079    \\
            \cifar   & 0.9973                    & 0.9989                   & 0.0028 $\pm$ 0.0095     & -0.0040 $\pm$ 0.0077   \\ \bottomrule
        \end{tabular}
    }
\end{table}

\paragraph{Training Difficulty}

Since \tool is merely a training technique, we can test \tool-trained models using a different classifier split. By design, if the training is successful, then under a given classifier split for testing, the model trained with the same split should have the best \tool accuracy. Although this is often true, we find that in some cases, a smaller classifier split results in higher \tool accuracy, indicating optimization issues.

We measure \tool accuracy for models trained with \ibp and \tool using different splits for \cifar (\cref{fig:trainability_cifar}) and \mnist (\cref{fig:trainability_mnist}). We observe that for \cifar $\epsilon=2/255$ and \mnist, the models trained and tested with the classifier/extractor split achieve the highest \tool accuracies, as expected, indicating a relatively well-behaved optimization problem. However, for \cifar $\epsilon=8/255$, the model trained with a classifier size of 1 achieves the highest \tool accuracy for all test splits and also the best adversarial and certified accuracy (see \cref{fig:split_location_8,tb:abalation_classifier_size_cifar}). This indicates that, in this setting, an earlier split and thus larger classifier component induces a (too) difficult optimization problem, leading to worse overall performance.
\begin{figure}[tp]
    \centering
    \begin{subfigure}[t]{0.48\textwidth}
        \centering
        \includegraphics[width=.96\linewidth]{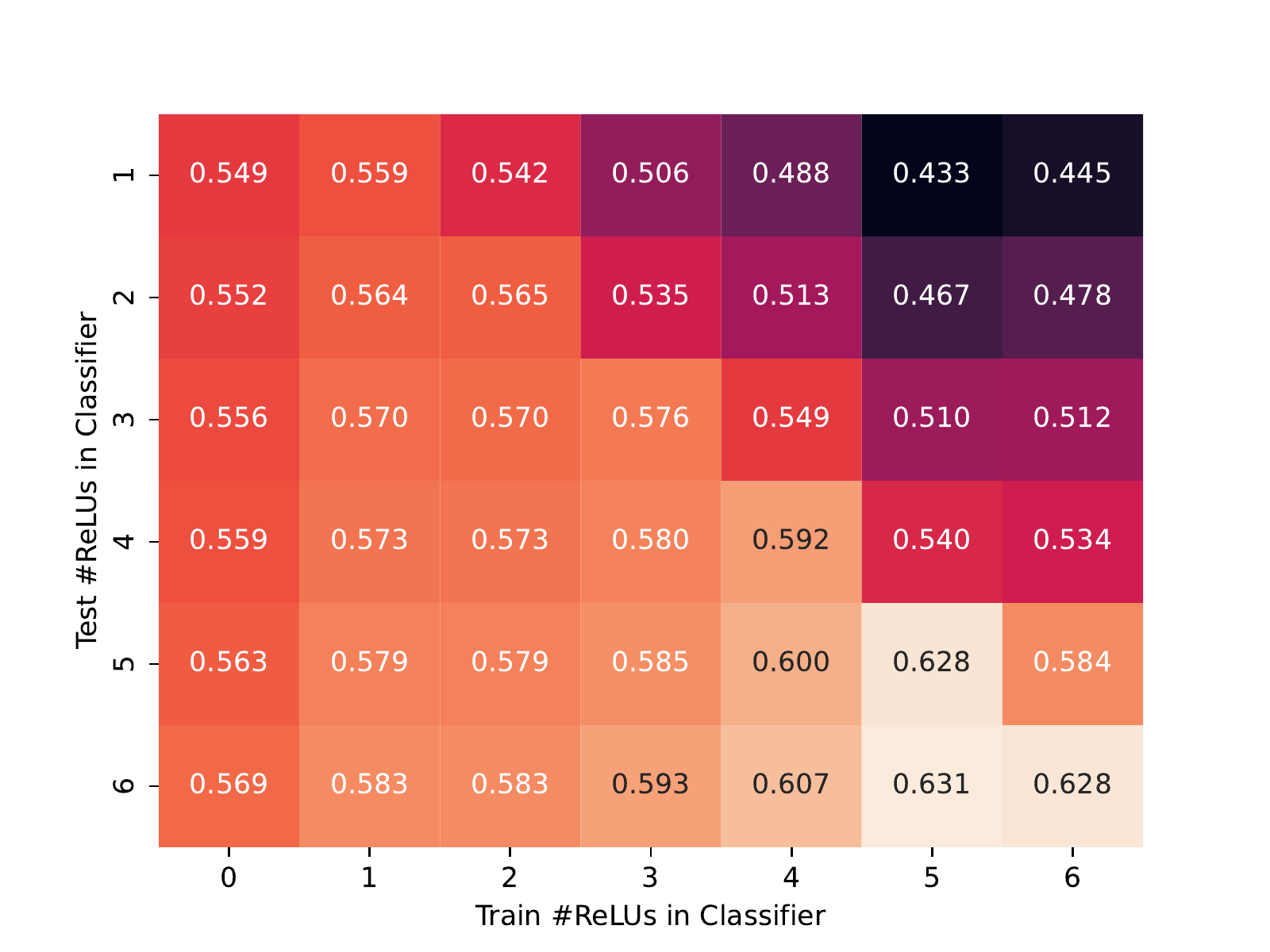}
        \caption*{ $\epsilon=\frac{2}{255}$}
    \end{subfigure}
    \begin{subfigure}[t]{0.48\textwidth}
        \centering
        \includegraphics[width=.96\linewidth]{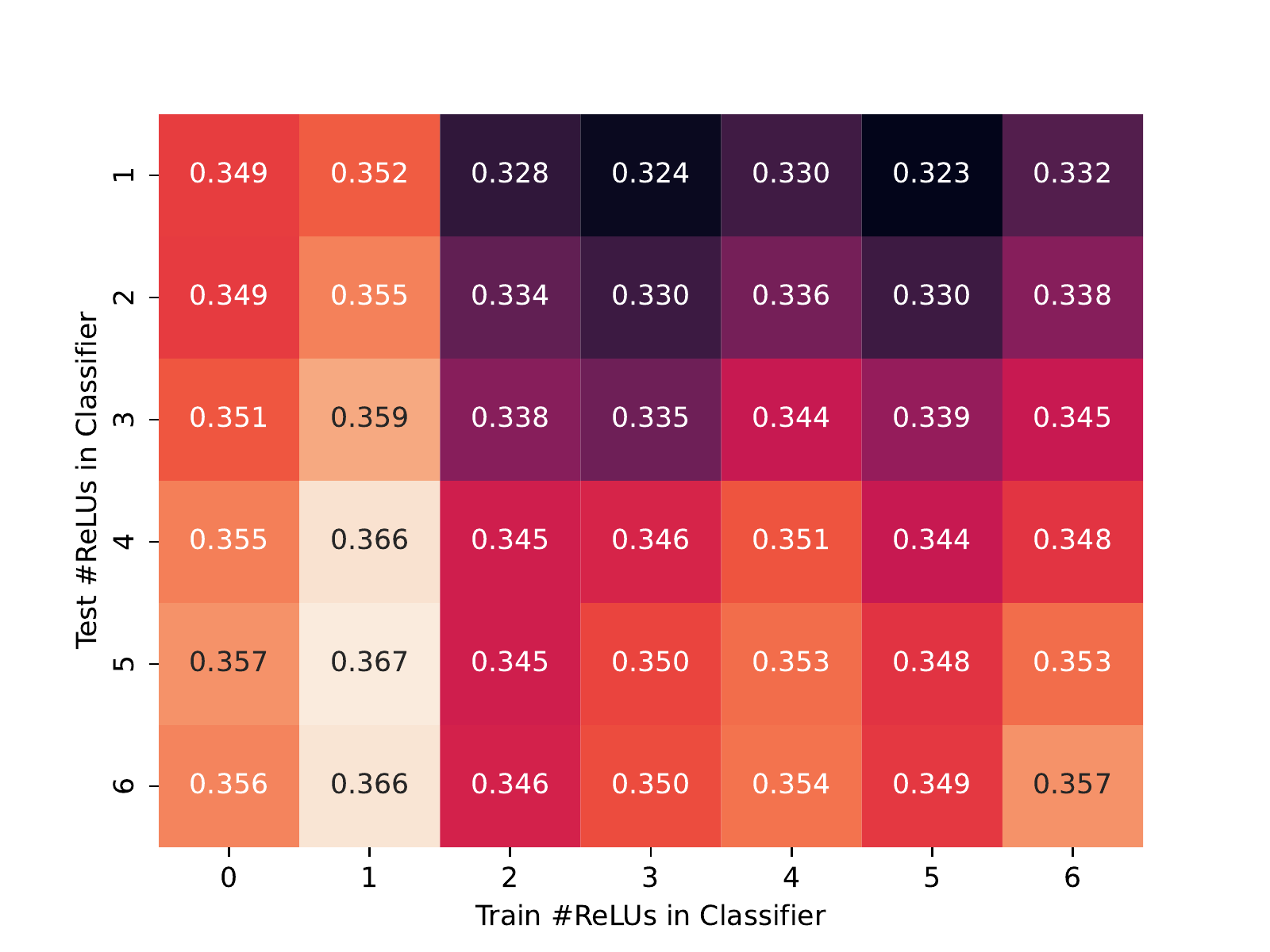}
        \caption*{ $\epsilon=\frac{8}{255}$}
    \end{subfigure}
    \caption{\tool accuracy of models trained with different classifier sizes for \cifar.}
    \label{fig:trainability_cifar}
\end{figure}

\begin{figure}[tp]
    \centering
    \begin{subfigure}[t]{0.48\textwidth}
        \centering
        \includegraphics[width=.96\linewidth]{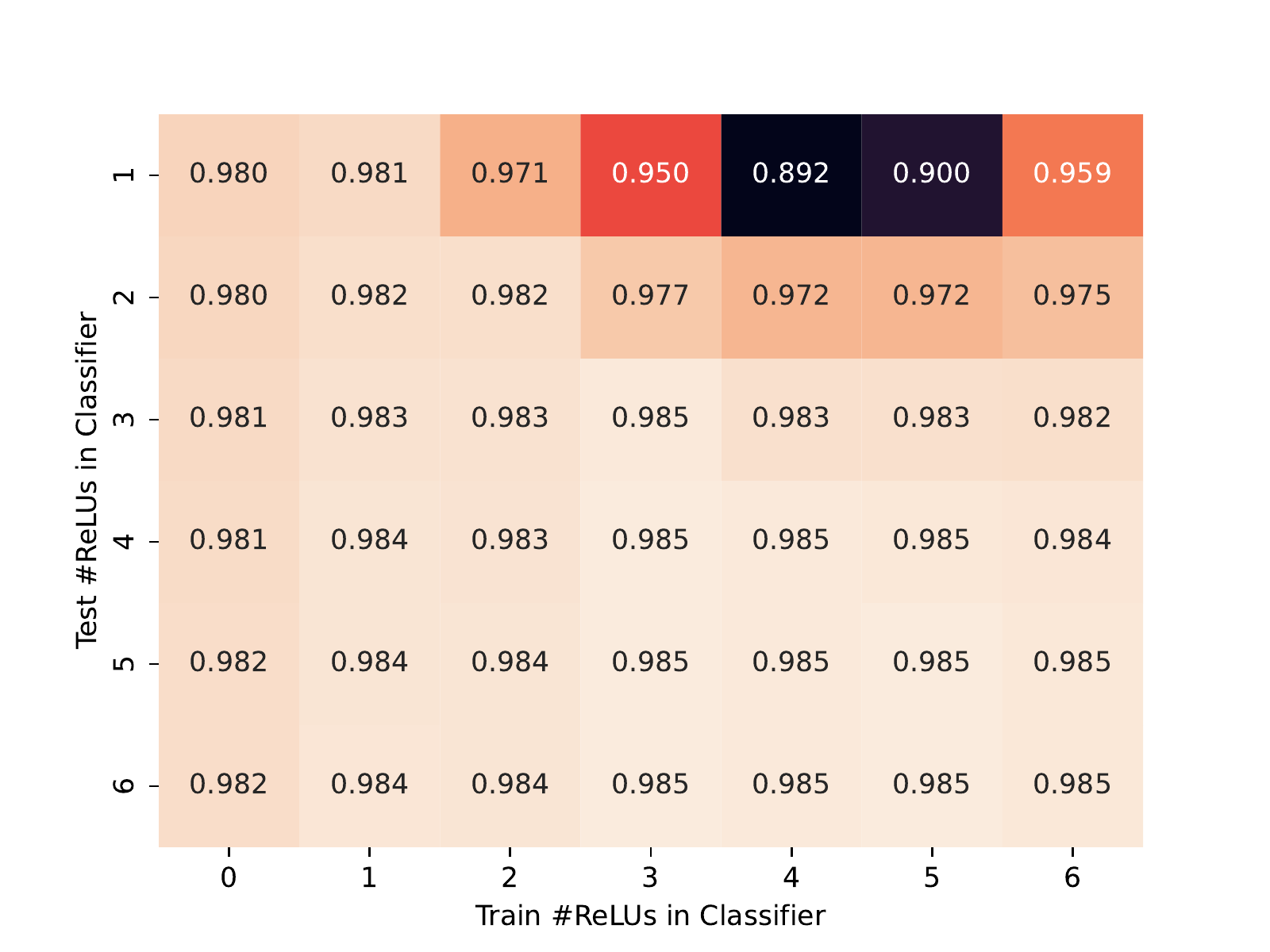}
        \caption*{ $\epsilon=0.1$}
    \end{subfigure}
    \begin{subfigure}[t]{0.48\textwidth}
        \centering
        \includegraphics[width=.96\linewidth]{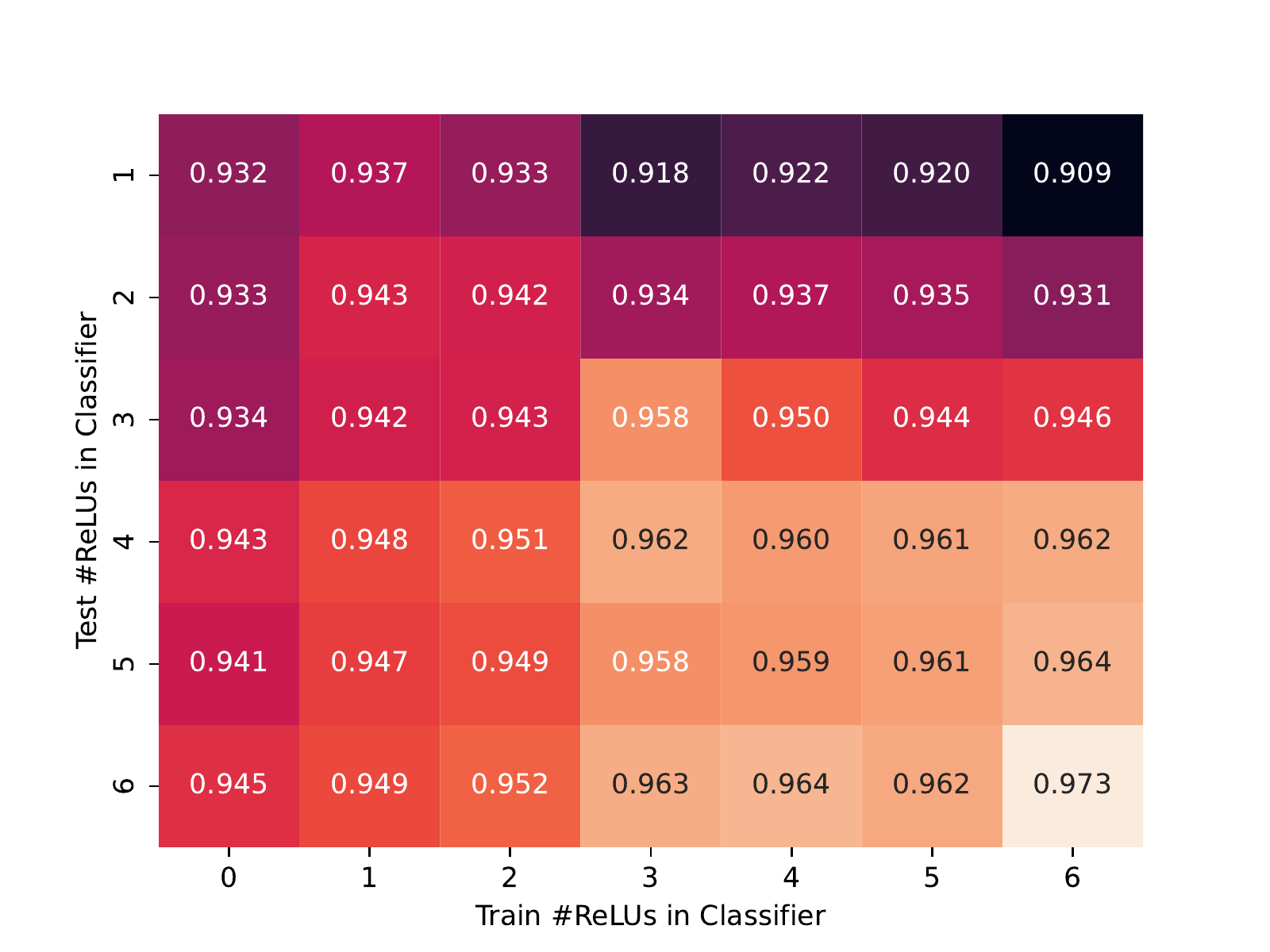}
        \caption*{ $\epsilon=0.3$}
    \end{subfigure}
    \caption{\tool accuracy of models trained with different classifier sizes for \mnist.}
    \label{fig:trainability_mnist}
\end{figure}

% \subsection{Additional Results}

\paragraph{Split Position}

We report detailed results for the experiment visualized in \cref{fig:split_location_8} (\cref{sec:split}) in \cref{tb:abalation_classifier_size_mnist} and \cref{tb:abalation_classifier_size_cifar}. We additionally report results on \TIN in \cref{tab:classifier_size_tin}.
% \begin{wraptable}[12]{r}{0.5\textwidth}
\begin{table}[t]
    \centering
    \caption{Effect of split position on accuracies $[\%]$ for fixed model size on \mnist.} \label{tab:classifier_size_mnist}
    \label{tb:abalation_classifier_size_mnist}
    \vspace{5pt}
    \resizebox{0.45\linewidth}{!}{
        \begin{tabular}{cccccc}
            \toprule
            \multirow{2.5}{*}{$\epsilon$} & \multirow{2.5}{*}{{\# ReLUs in Classifier}} & \multirow{2.5}{*}{{Nat.}} & \multirow{2.5}{*}{{Adv.}} & \multicolumn{2}{c}{{Cert.}}\\
            \cmidrule(lr){5-6}
            &&&& {MN-BaB}& {IBP}\\ 
            \midrule
            \multirow{7}{*}{0.1} & 0 & 98.87& 98.16& 98.13& 97.83\\
            & 1 & 99.06& 98.37& 98.31& 96.27 \\
            & 2 & 99.16& 98.35& 98.25& 87.82 \\
            & 3& 99.19& \textbf{98.51}& \textbf{98.39}& 62.83\\
            & 4& \textbf{99.28}& 98.47& 98.03& 4.75 \\
            & 5& 99.22& \textbf{98.51}& 98.17& 9.76 \\
            & 6& 99.09& 98.45& 98.27& 81.89\\ 
            \cmidrule(lr){1-6}
            \multirow{7}{*}{{0.3}} & 0 & 97.60& 93.37& 93.15& 93.08 \\
            & 1 & 97.94& 94.01& \textbf{93.62}& 92.76\\
            & 2 & 98.16& 94.18& 93.55& 91.85 \\
            & 3& 98.63& 94.48& 93.03& 89.40\\
            & 4& 98.7 & 94.85& 93.44& 89.52\\
            & 5& 98.63& 94.64& 93.26& 89.15\\
            & 6& \textbf{98.88}& \textbf{95.11}& 92.70& 85.03 \\
            \bottomrule
        \end{tabular}
    }
\end{table}
\begin{table}
    \centering
    \caption{Effect of split position on accuracies $[\%]$ for fixed model size on \cifar.}
    \label{tb:abalation_classifier_size_cifar}
    \vspace{5pt}
    \resizebox{0.45\linewidth}{!}{
        \begin{tabular}{cccccc}
            \toprule
            \multirow{2.5}{*}{$\epsilon$} & \multirow{2.5}{*}{{\#ReLUs}} & \multirow{2.5}{*}{{Nat.}} & \multirow{2.5}{*}{{Adv.}} & \multicolumn{2}{c}{{Cert.}}\\
            \cmidrule(lr){5-6}
            &&&& {MN-BaB}& {IBP} \\ 
            \midrule
            \multirow{7}{*}{{$\frac{2}{255}$}} & 0 & 67.27& 56.32& 56.14& 53.54\\
            & 1 & 70.10& 57.78& 57.48& 41.86\\
            & 2 & 70.74& 57.83& 57.39& 40.24\\
            & 3& 71.88& 58.89& 58.23& 34.41 \\
            & 4& 72.45& 60.38& 59.47& 31.88\\
            & 5& 75.09& \textbf{63.00}& \textbf{61.56}& 24.36\\
            & 6& \textbf{75.40}& 62.73& 61.11& 24.90\\ 
            \cmidrule(lr){1-6}
            \multirow{7}{*}{{$\frac{8}{255}$}} & 0 & 48.15& 34.63& 34.60& 34.26 \\
            & 1 & 49.76& \textbf{35.29}& \textbf{35.10}& 32.92 \\
            & 2 & 47.28& 33.54& 33.12& 28.94 \\
            & 3& 48.76& 33.50& 33.12& 29.14\\
            & 4& 50.19& 34.78& 34.35& 29.14\\
            & 5& 50.2 & 34.33& 33.72& 28.83\\
            & 6& \textbf{51.03}& 35.25& 34.44& 29.97 \\
            \bottomrule
        \end{tabular}
    }
\end{table}
% \end{wraptable}

% \multirow{7}{*}{{2/255}} & IBP & 67.27& 56.32& 56.14& 53.54& +1.63& +0.97 \\
% & 4 & 70.10& 57.78& 57.48& 41.86& +11.24 & +4.38 \\
% & 8 & 70.74& 57.83& 57.39& 40.24& +12.31 & +4.84 \\
% & 11& 71.88& 58.89& 58.23& 34.41& +15.71 & +8.11 \\
% & 14& 72.45& 60.38& 59.47& 31.88& +17.18 & +10.41\\
% & 17& 75.09& \textbf{63.00}& \textbf{61.56}& 24.36& +20.68 & +16.52\\
% & 20& \textbf{75.40}& 62.73& 61.11& 24.90& +20.58 & +15.63\\ 
% \cmidrule(lr){1-8}
% \multirow{7}{*}{{8/255}} & IBP & 48.15& 34.63& 34.60& 34.26& +0.19& +0.15 \\
% & 4 & 49.76& \textbf{35.29}& \textbf{35.10}& 32.92& +1.46& +0.72 \\
% & 8 & 47.28& 33.54& 33.12& 28.94& +2.83& +1.35 \\
% & 11& 48.76& 33.50& 33.12& 29.14& +2.47& +1.51 \\
% & 14& 50.19& 34.78& 34.35& 29.14& +3.18& +2.03 \\
% & 17& 50.2 & 34.33& 33.72& 28.83& +2.69& +2.20 \\
% & 20& \textbf{51.03}& 35.25& 34.44& 29.97& +2.65& +1.82 \\

\paragraph{IBP Regularization} We repeat the experiment reported on for \mnist in \cref{tab:ibp_reg} for \TIN, presenting results in \cref{tab:ibp_reg_tin}. We generally observe the same trends, although peak certification performance is achieved slightly later at $\w_{\text{\tool}}=10$ instead of $\w_{\text{\tool}}=5$.

% \paragraph{Tightness Coefficient for Different Width}

% We repeat our experiment in \cref{sec:empirical_PI} for fully connected 3-layer network with different width and show the result in \cref{fig:PIC_width}. The narrow and wide network has $1/4$ and $4$ times of the width of the medium network, respectively. \cref{sec:empirical_PI} shows the trend is consistent for all the models, further supporting our discussion in \ref{sec:empirical_PI}. In addition, \cref{sec:empirical_PI} shows that larger networks, thus having more capacity, requires less propagation invariance. This suggests that larger models are able to handle the additional regularization introduced by spurious points and thus need less propagation invariance.

\begin{table}[t]
    \centering
    \caption{Effect of split position on accuracies $[\%]$ for fixed model size on \TIN.}\label{tab:classifier_size_tin}
    \resizebox{.99\linewidth}{!}{
        \begin{tabular}{ccccccccccc}
            \toprule
            \multirow{2}{*}{ReLU} & \multicolumn{5}{c}{TAPS} & \multicolumn{5}{c}{STAPS}                                                                                                       \\
                                  & Nat. (\%)                & Adv. (\%)                 & Cert. (\%) & Train (s) & Certify (s) & Nat. (\%) & Adv. (\%) & Cert. (\%) & Train (s) & Certify (s) \\ \midrule
            1                     & 28.34                    & 20.94                     & 20.82      & 306$\,$036    & 23$\,$497       & 28.75     & 22.25     & 22.04      & 350$\,$924    & 35$\,$894       \\
            2                     & 27.02                    & 20.94                     & 20.84      & 944$\,$520    & 32$\,$407       & 28.98     & 22.40     & 22.16      & 861$\,$639    & 35$\,$183       \\ \bottomrule
        \end{tabular}
    }
\end{table}

% \begin{wraptable}[14]{r}{0.5 \textwidth}
\begin{table}[t]
    \centering
    % \vspace{-4mm}
    \caption{Effect of \ibp regularization and the \tool gradient expanding coefficient $\alpha$ for \TIN $\epsilon=\frac{1}{255}$.}\label{tab:ibp_reg_tin}
    % \vspace{-1mm}
    \begin{adjustbox}{width=0.45\linewidth,center}
        \begin{threeparttable}
            \begin{tabular}{@{}ccccccc@{}}
                \toprule
                $w_{\text{\tool}}$   & {Avg time (s)} & Nat. (\%) & {Adv. (\%)} & {Cert. (\%)} \\
                \midrule
                {$\L_{\text{\ibp}}$} & 0.28           & 25.00    & 19.72       & 19.72        \\
                {1}                  & 1.17           & 25.83    & 20.24       & 20.22        \\
                {5}                  & 2.34           & 28.34    & 20.94       & 20.82        \\
                {10}                 & 4.12           & 28.23    & \textbf{21.05}       & \textbf{20.89}        \\
                {20}                 & 5.94           & \textbf{28.44}    & 20.68       & 20.44        \\
                \bottomrule
            \end{tabular}
        \end{threeparttable}
    \end{adjustbox}
    % \vspace{-1mm}
\end{table}
% \end{wraptable}

\paragraph{Repeatability}
Due to the large computational cost of up to 10 GPU-days for some experiments (see \cref{tb:time_TAPS,tb:time_STAPS}), we could not repeat all experiments multiple times to report full statistics. However, we report statistics for the best-performing method for \mnist at $\epsilon=0.1$ and $\epsilon=0.3$ and \cifar at $\epsilon=2/255$ and $\epsilon=8/255$ (see  \cref{tab:mnist_repeatability}). We generally observe small standard deviations, indicating good repeatability of our results.

\begin{table}[]
    \centering
    \caption{Mean and standard deviation (over three repeats) for the method with best certified accuracy.}\label{tab:mnist_repeatability}
    \begin{tabular}{@{}lclccc@{}}
        \toprule
        Dataset & $\epsilon_\infty$ & Method  & Nat. [\%]  & Adv. [\%] & Cert. [\%] \\
        \midrule
        \multirow{2.2}*{\mnist} & 0.1 & \tool & 99.22 $\pm$ 0.03 & 98.45 $\pm$ 0.06 & 98.28 $\pm$ 0.10 \\
                                & 0.3 & \tool & 97.96 $\pm$ 0.04 & 93.96 $\pm$ 0.04 & 93.57 $\pm$ 0.02 \\ 
        \cmidrule(lr){1-2}
        \multirow{2}*{\cifar}   & 2/255 & \toolp & 79.75 $\pm$ 0.23 & 65.91 $\pm$ 0.12 & 62.72 $\pm$ 0.23\\
                                & 8/255 & \tool  & 49.07 $\pm$ 0.61 & 34.75 $\pm$ 0.47 & 34.57 $\pm$ 0.46\\
        \bottomrule
    \end{tabular}
\end{table}

\section{Limitations}

\tool and all other certified training methods can only be applied to mathematically well-defined perturbations of the input such as $\ell_p$-balls, while real-world robustness may require significantly more complex perturbation models. 
Further and similarly to other unsound certified training methods, \tool introduces a new hyperparameter, the split position, that can be tuned to improve performance further beyond the default choice of 1 ReLU layer in the classifier. 
Finally, while training with \tool is similarly computationally expensive as with other recent methods, it is notably more computationally expensive than simple certified training methods such as \ibp.

\section{Reproducibility}
We publish our code, trained models, and detailed instructions on how to reproduce our results at \texttt{\href{https://github.com/eth-sri/taps}{github.com/eth-sri/taps}}, providing an anonymized version to the reviewers\footnote{We provide the codebase with the supplementary material, including instructions on how to download our trained models.}. 
Additionally, we provide detailed descriptions of all hyper-parameter choices, data sets, and preprocessing steps in \cref{sec:exp_setting}.
}{}

%%%%%%%
% END %
%%%%%%%
% this message marks the end of the document (important for splitting of paper)
\message{^^JLASTPAGE \thepage^^J}

\end{document}